\documentclass[sigconf]{aamas} 

\usepackage{balance} 
\usepackage{booktabs}
\usepackage{amsmath}
\usepackage{amsfonts}
\usepackage{mathtools}
\usepackage{amsthm}
\usepackage{bm}
\usepackage{subcaption}
\usepackage{xspace}
\usepackage{algorithm}
\usepackage{algorithmic}
\usepackage[toc,page]{appendix}

\setcopyright{ifaamas}
\acmConference[AAMAS '23]{Proc.\@ of the 22nd International Conference
on Autonomous Agents and Multiagent Systems (AAMAS 2023)}{May 29 -- June 2, 2023}
{London, United Kingdom}{A.~Ricci, W.~Yeoh, N.~Agmon, B.~An (eds.)}
\copyrightyear{2023}
\acmYear{2023}
\acmDOI{}
\acmPrice{}
\acmISBN{}

\acmSubmissionID{562}

\title[AAMAS-2023 Formatting Instructions]{\name: Multi-Agent Experience Replay via Collective Priority Optimization}

\author{Yongsheng Mei}
\affiliation{
  \institution{The George Washington University}
  \city{Washington, DC}
  \country{United States}}
\email{ysmei@gwu.edu}

\author{Hanhan Zhou}
\affiliation{
  \institution{The George Washington University}
  \city{Washington, DC}
  \country{United States}}
\email{hanhan@gwu.edu}

\author{Tian Lan}
\affiliation{
	\institution{The George Washington University}
	\city{Washington, DC}
	\country{United States}}
\email{tlan@gwu.edu}

\author{Guru Venkataramani}
\affiliation{
	\institution{The George Washington University}
	\city{Washington, DC}
	\country{United States}}
\email{guruv@gwu.edu}

\author{Peng Wei}
\affiliation{
	\institution{The George Washington University}
	\city{Washington, DC}
	\country{United States}}
\email{pwei@gwu.edu}

\begin{abstract}
	Experience replay is crucial for off-policy reinforcement learning (RL) methods. By remembering and reusing the experiences from past different policies, experience replay significantly improves the training efficiency and stability of RL algorithms. Many decision-making problems in practice naturally involve multiple agents and require multi-agent reinforcement learning (MARL) under centralized training decentralized execution paradigm. Nevertheless, existing MARL algorithms often adopt standard experience replay where the transitions are uniformly sampled regardless of their importance. Finding prioritized sampling weights that are optimized for MARL experience replay has yet to be explored. To this end, we propose \name, which formulates optimal prioritized experience replay for multi-agent problems as a regret minimization over the sampling weights of transitions. Such optimization is relaxed and solved using the Lagrangian multiplier approach to obtain the close-form optimal sampling weights. By minimizing the resulting policy regret, we can narrow the gap between the current policy and a nominal optimal policy, thus acquiring an improved prioritization scheme for multi-agent tasks. Our experimental results on Predator-Prey and StarCraft Multi-Agent Challenge environments demonstrate the effectiveness of our method, having a better ability to replay important transitions and outperforming other state-of-the-art baselines.
\end{abstract}


\newcommand{\BibTeX}{\rm B\kern-.05em{\sc i\kern-.025em b}\kern-.08em\TeX}

\newtheorem{theorem}{Theorem}
\newtheorem{lemma}{Lemma}
\newtheorem{remark}{Remark}
\newtheorem{definition}{Definition}
\newtheorem{assumption}{Assumption}

\newcommand{\p}{\mathop{}\!{\partial}}
\newcommand{\T}{^\mathrm{T}}

\newcommand{\bu}{\mathbf{u}} 
\newcommand{\btau}{\bm{\tau}} 
\newcommand{\bpi}{\bm{\pi}} 
\newcommand{\name}{MAC-PO\xspace}
\newcommand{\eqdef}{\stackrel{\mathsf{def}}{=}}

\begin{document}


\pagestyle{fancy}
\fancyhead{}


\maketitle 


\section{Introduction}

Reinforcement learning (RL) has demonstrated great success in solving challenging problems~\cite{kalashnikov2018scalable, shani2005mdp}. For off-policy RL, experience replay mechanism~\cite{lin1992self, mnih2015human} allows utilizing history experiences in the replay buffer that stores the most recently collected transitions for training. It has been shown to significantly improve policy learning and RL algorithms' stability. Due to these benefits, various approaches~\cite{schaul2016prioritized, zha2019experience, sinha2022experience} for computing priority scores of experiences have been proposed for single-agent RL. For instance, prioritized experience replay (PER)~\cite{schaul2016prioritized} leverages predefined metrics for prioritizing experience based on the temporal-difference (TD) error related to the loss of the critic network. It calculates the sampling probabilities proportional to the magnitude of TD error, resulting in a non-uniform sampling/prioritization scheme in Q-learning~\cite{watkins1992q}. 

In practice, we often face RL tasks involving multiple agents sharing the same environment, e.g., in autonomous driving~\cite{cao2012overview, hu2019interaction} and robotics and planning~\cite{matignon2012coordinated, levine2016end, huttenrauch2017guided}. To coordinate multiple agents and learn desired joint behavior from their collective experiences, we require multi-agent reinforcement learning (MARL)~\cite{vinyals2019grandmaster, jaques2019social, baker2019emergent}, such as value-based methods QMIX~\cite{rashid2018qmix} and QPLEX~\cite{wang2020qplex}, or policy-based methods COMA~\cite{foerster2018counterfactual} and MADDPG~\cite{lowe2017multi}. These approaches leverage centralized training decentralized execution (CTDE)~\cite{kraemer2016multi} and often employ standard memory replay buffers with a uniform sampling of transition history. However, in MARL problems, such a standard sampling strategy of the replay buffer cannot reflect the dynamics in the environment caused by multi-agent interactions. Therefore, indiscriminately training from past experiences will make agents less capable of using experiences optimally. Although we can impart existing single-agent prioritization schemes directly to the joint action-value function of MARL, such a naive application is oblivious to the interaction between multiple agents in the shared environment and may lead to sub-optimal performance. Thus, priority optimization for experience replay in MARL is still an open problem.

To this end, we propose \name, which formulates MARL prioritized experience replay problems as a regret minimization over the sampling weights of different state-action values. Specifically, we define policy regret as the difference between the expected discounted reward of a nominal optimal policy and that of the current policy under given sampling weights. By minimizing such a policy regret by considering its upper bound, we can narrow the gap between the optimal and current policies with respect to the sampling weights, leading to an optimal solution of sampling weights with minimum regret. We note that similar regret minimization techniques have been employed in single-agent RL settings~\cite{liu2021regret}. Our paper expands it to analyze multi-agent prioritized experience replay and develops new solutions, e.g., to handle joint actions of multiple agents and to analyze concurrent optimality constraints. It turns out that the optimal sampling weights in MARL now depend on the collective policies of decentralized agents. To the best of our knowledge, this is the first proposal for optimizing prioritized experience replay in cooperative MARL.

In particular, we show that the proposed optimization can be solved via the Lagrangian multiplier method~\cite{bertsekas2014constrained} considering an upper bound of the regret. Since we focus on multi-agent prioritized experience replay problems, the optimization objective is defined by the joint policy of all agents. Therefore, when we further analyze the Lagrangian conditions for optimality, the agents' conditions depend on each other and become a vector form. Further, by examining a weighted Bellman equation, we leverage the implicit function theorem~\cite{krantz2002implicit} for multiple agents and apply a group of Karush–Kuhn–Tucker (KKT)~\cite{ghojogh2021kkt} conditions to find the optimal sampling weights in closed form. 

Our results illuminate the key principles contributing to optimal sampling weights in multi-agent prioritized experience replay. The optimal sampling weights can be interpreted to consist of four components: Bellman error, value enhancement, on-policiness of available transitions, and a new term depending on joint action probabilities. While the first three have been identified in single agent settings~\cite{liu2021regret, kumar2020discor}, our paper shed light on a new term - as a function of joint action probabilities - to reveal that optimal sampling weights of multi-agent prioritized experience replay should depend on the interaction among all the agents within an environment. More specifically, we should assign the highest sampling weights to transitions only if one agent's action probability is small in the transition while all other agents' action probabilities are large. The result -- slightly counter-intuitive since higher weights are assigned to transitions with more differentiated action probabilities (rather than similar ones) -- is quantified and formalized as a new theorem in our paper. Based on this result, we also present an approximated solution for estimating sampling weights in problems involving many agents or having limited information for an exact solution.

Following the theoretical analysis, we propose a MARL algorithm, \name, for multi-agent prioritized experience replay via regret minimization. Like existing methods, \name can be plugged into any MARL algorithms with a memory replay buffer. We validate the effectiveness of \name in StarCraft Multi-Agent Challenge (SMAC)~\cite{samvelyan2019starcraft} and Predator-Prey~\cite{bohmer2020deep} through comparison with other single-agent experience replay methods (adapted to MARL problems by considering all agents as a conceptual agent). Moreover, we also compared \name with state-of-the-art MARL algorithms. In the experiments, \name demonstrates improved convergence and superior empirical performance.

The main contributions of our work are as follows:
\begin{itemize}
	\item We propose a novel method, \name, which formulates multi-agent experience replay as a policy regret minimization and solves the optimal sampling weights in closed form.
	\item The theoretical results illuminate a new factor in optimal sampling weights and motivate the design of new MARL experience replay algorithms with both exact and approximated weights.
	\item Experiment results of \name in SMAC and Predator-Prey environments demonstrate superior convergence and empirical performance over various baselines, including experience replay and state-of-the-art MARL methods.
\end{itemize}

\section{Background}

\subsection{Partially Observable Markov Decision Process}

In this work, we consider a multi-agent sequential decision-making task as a decentralized partially observable Markov decision process (Dec-POMDP)~\cite{oliehoek2016concise} consisting of a tuple $ G=\langle S,U,P,R,Z,O,n,\gamma \rangle $, where $ s \in S $ describes the global state of the environment. At each time step, each agent $ a \in A \equiv \{ 1,\dots,n \} $ selects an action $ u_a \in U $, and all selected actions combine and form a joint action $ \mathbf{u} \in \mathbf{U} \equiv U^n $. Such a process leads to a transition in the environment based on the state transition function $ P(s'|s,\mathbf{u}):S \times \mathbf{U} \times S \rightarrow [0,1] $. All agents share the same reward function $ r(s,\mathbf{u}):S \times \mathbf{U} \rightarrow \mathbb{R} $ with a discount factor $ \gamma \in [0,1) $.

In the partially observable environment, the agents' individual observations $ z \in Z $ are generated by the observation function $ O(s,u):S \times A \rightarrow Z $. Each agent has an action-observation history $ \tau_a \in T \equiv (Z \times U)^*$. Conditioning on the history, the policy becomes $ \pi^a(u_a|\tau_a):T \times U \rightarrow [0,1] $. The joint policy $ \bpi $ has a joint action-value function: $ Q^{\bpi}(s_t, \mathbf{u}_t)=\mathbb{E}_{s_{t+1:\infty},\mathbf{u}_{t+1:\infty}}[R_t|s_t,\mathbf{u}_t] $, where $ t $ is the timestep and $R_t=\sum_{i=0}^{\infty} \gamma^i r_{t+i}$ is the discounted return. In this paper, we adopt the CTDE mechanism. The learning algorithm has access to all local action-observation histories $ \btau $ and global state $ s $ during training, yet every agent can only access its individual history in execution. Although we compute individual policy based on histories in practice, following the existing work~\cite{su2022divergence}, we will use $ \pi^a(u_a|s) $ in analysis and proofs for simplicity.

\subsection{Policy Regret}

In MARL, we aim to find a joint policy $ \bpi $ that can maximize the expected return: $ \eta(\bpi)=\mathbb{E}_{\bpi}[\sum_{i=0}^{\infty} \gamma^i r_{t+i}] $. For a fixed policy, the Markov decision process becomes a Markov reward process, where the discounted state distribution is defined as $ d^{\bpi}(s) $. Similarly, the discounted state-action distribution is defined as $ d^{\bpi}(s,\bu)=d^{\bpi}(s)\bpi(\bu|s) $. Then, we will have the expected return rewritten as $ \eta(\bpi)=\frac{1}{1-\gamma} \mathbb{E}_{d^{\bpi}(s,\mathbf{u})}[r(s,\mathbf{u})] $.

We assume a nominal optimal joint policy $ \bpi^* $ such that $\bpi^*=\arg\max_{\bpi}\eta(\bpi) $. The regret of the joint policy $ \bpi $ is the difference between the expected discounted reward of an optimal policy and that of the current policy as $ \textrm{regret}(\bpi)=\eta(\bpi^*)-\eta(\bpi) $. The policy regret measures the expected loss when following the current policy $ \bpi $ instead of optimal policy $ \bpi^* $. Since $ \eta(\bpi^*) $ is a constant, minimizing the regret is consistent with maximizing of expected return $ \eta(\bpi) $. In this paper, we use regret as an alternative optimization objective for finding the optimal sampling weight in MARL tasks, along with multiple constraints, such as the Bellman equation. By minimizing the regret, the current joint policy $ \bpi_k $ of all agents' actions will approach the optimum $ \bpi^* $.

\subsection{Connection of Prioritized Sampling and Weighted Loss Function}

The design of prioritized sampling methods is not isolated from the loss function. Instead, the expected gradient of a loss function with non-uniform sampling is equivalent to that of a weighted loss function with uniform sampling, which facilitates the design of prioritized sampling algorithms~\cite{fujimoto2020equivalence}. Given a data sample set $ D $ of size $ d $, a regular loss function $ L_1 $ where we use a specific priority scheme $ pr(\cdot) $ to sample the transitions, and another loss function $ L_2 $ whose transitions are sampled uniformly, the two approaches are equivalent if we have the following requisition satisfied: $$ \nabla_Q L_1 = \frac{\chi}{pr}\nabla_Q L_2, $$ where $ \chi=\frac{\sum_{i}pr(i)}{d} $ and $ i \in D $ is the uniformly sampled instance.

We can leverage such equivalence to analyze the correctness of approaches using non-uniform sampling by transforming the loss into the uniform-sampling equivalent or considering whether the new loss is in line with the target objective. It also provides a recipe for transforming a regular loss function $ L_1 $ with a non-uniform sampling scheme into an equivalent weighted loss function $ L_2 $ with uniform sampling.

\section{Related Works}
\label{sec:related}

\subsection{MARL Algorithms}

MARL algorithms have developed into neural-network-based methods that can cope with high-dimensional state and action spaces. Early methods practice finding policies for a multi-agent system by directly learning decentralized value functions or policies. For example, independent Q-learning~\cite{tan1993multi} trains independent action-value functions for each agent via Q-learning. \cite{tampuu2017multiagent} extends this technique to DQN~\cite{mnih2015human}. Recently, approaches for CTDE have come up as centralized learning of joint actions that can conveniently solve coordination problems without introducing non-stationary. COMA~\cite{foerster2018counterfactual} uses a centralized critic to train decentralized actors to estimate a counterfactual advantage function for every agent. Similar works~\cite{gupta2017cooperative, lowe2017multi} are also proposed based on such analysis. Under CTDE manner, value decomposition approaches~\cite{guestrin2002coordinated, castellini2019representational} are widely used in value-based MARL. Such methods integrate each agent's local action-value functions through a learnable mixing function to generate global action values. For instance, QMIX~\cite{rashid2018qmix} estimates the optimal joint action-value function by combining mentioned utilities via a continuous state-dependent monotonic function generated by a feed-forward mixing network with non-negative weights. QTRAN~\cite{son2019qtran} and QPLEX~\cite{wang2020qplex} further extend the class of value functions that can be represented. ReMIX~\cite{mei2023remix} provides a factorization weighting scheme to find the optimal projection of an unrestricted mixing function onto monotonic function classes. PAC\cite{zhou2022pac} and LAS-SAC\cite{zhou2022value} proposes to use latent assisted information~\cite{mei2023exploiting} as extra-state information for better value factorization. Aside from methods focusing on tackling cooperative problems, other mechanisms can also solve competitive problems or mixed problems. MADDPG~\cite{lowe2017multi} utilizes the ensemble of policies for each agent that leads to more robust multi-agent policies, showing strength in cooperative and competitive scenarios. Beyond that, the extensions~\cite{iqbal2019actor, su2021value, gogineni2023scalability} of MADDPG have been proposed to realize further optimization towards the original algorithm. In this paper, we focus on the cooperative setting and leverage a standard QMIX with a monotonic mixing network, along with an unrestricted QMIX~\cite{rashid2020weighted} without a monotonic function for retrieving the optimal joint policy.

\subsection{Single-Agent Experience Replay}

Many RL algorithms adopt prioritization to increase the learning speed, initially originating from prioritized sweeping for value iteration~\cite{moore1993prioritized, van2013planning}. Besides, they have also been used in other modern applications, such as learning from demonstrations~\cite{hester2018deep}. Prioritized experience replay~\cite{schaul2016prioritized} is one of several popular improvements to the DQN algorithms~\cite{van2016deep, wang2016dueling} and has been included in many algorithms combining multiple improvements~\cite{horgan2018distributed, barth2018distributed}. Variations of PER have been proposed for considering sequences of transitions~\cite{daley2019reconciling, brittain2019prioritized} or optimizing the prioritization function~\cite{zha2019experience}. Furthermore, to favor recent transitions without explicit prioritization, alternate replay buffers have been raised~\cite{novati2019remember}. \cite{de2015importance, zhang2017deeper} studied the composition and size of the replay buffer, and \cite{liu2018effects} looked into prioritization in simple environments. Other important sampling approaches also greatly improved the performance. \cite{kumar2020discor} re-weights updates to reduce variance. \cite{liu2021regret} uses the regret minimization method to design the prioritized experience replay scheme for the only agent in the environment. MaPER~\cite{oh2021model} employs model learning to improve experience replay by using a model-augmented critic network and modifying the rule of priority. Also, new loss function designs can help develop prioritization schemes~\cite{sujit2022prioritizing}. So far, most works about experience replay are designed for single-agent reinforcement learning, and a limited number of works~\cite{wang2019experience, fan2020prioritized, ahilan2021correcting} investigate the possible extensions. In this paper, we proposed \name for MARL tasks by considering the interaction among multiple agents through collective priority optimization to seek an optimal multi-agent prioritization mechanism.

\section{Methodology}

\subsection{Problem Formulation}

Let $ Q_k $ denote the action-value function at iteration $ k $. We leverage $ \mathcal{B}^*Q_{k-1} $ as the target with a Bellman operator $\mathcal{B}^*$ and update $ Q_k $ in tandem using a weighted Bellman equation: $Q_k=\arg\min_{Q \in \mathcal{Q}} \mathbb{E}_{\mu}[w_k(s,\mathbf{u})(Q-\mathcal{B}^*Q_{k-1})^2(s,\mathbf{u})]$, where $w_k(s,\mathbf{u})$ represent non-negative sampling weights for different transitions that need to be optimized for the experience replay.

To formulate the policy regret with respect to the joint action-value function, we consider a Boltzmann policy $\bpi_k$ corresponding to each agent's individual utilities $ Q^a_k $, i.e., $ \bpi_k=[\pi^1_k,...,\pi^n_k]\T $ and $ \pi^a_k={e^{Q^a_k(\tau_a,u_a)}}/{\sum_{\tau_a,u_a'} e^{Q^a_k(\tau_a,u_a')}} $. Our objective is to minimize the policy regret $ \eta(\bpi^*)-\eta(\bpi) $ over non-negative sampling weights under relevant constraints, i.e., 
\begin{equation}
	\begin{aligned}
		\min_{w_k} \quad &\eta(\bpi^*)-\eta(\bpi_k) \\
		\textrm{s.t.} \quad &Q_k=\arg\min_{Q \in \mathcal{Q}} \mathbb{E}_{\mu}[w_k(s,\bu)(Q-\mathcal{B}^*Q_{k-1})^2(s,\bu)], \\
		&\mathbb{E}_{\mu}[w_k(s,\bu)]=1, \quad w_k(s,\bu) \ge 0,
	\end{aligned}
	\label{eq:obj_1}
\end{equation}
where $ \bpi_k $ and $ \bpi_k^* $ are Boltzmann policies for the current and nominal optimal policy, and the latter can be obtained from another network. The sampling weights must sum up to 1, and $ \mu $ is the distribution that we uniformly sample data from the replay buffer. An additional table to summarize and explain the common notations is provided in Appendix~\ref{subsec:notation}.

\subsection{Solving Optimal Sampling Weights for Experience Replay}

Our goal is to seek the optimal priority by minimizing the regret at every iteration $ k $, with respect to the weight $ w $ used for Bellman error minimization at iteration $ k $. For this purpose, we consider an upper bound of the relaxed regret objective and formulate its Lagrangian by introducing Lagrangian multipliers regarding the constraints. It allows us to solve the proposed regret-minimization problem and obtain optimal projection weights in closed form (albeit with a normalization factor $ Z^* $). 

\begin{theorem}[Optimal sampling weight]
	The optimal weight $ w_k(s,\mathbf{u}) $ to a relaxation of the regret minimization problem in Equation~\eqref{eq:obj_1} with discrete action space is given by:
	\begin{equation}
		w_k(s,\mathbf{u})=\frac{1}{Z^*}(E_k(s,\mathbf{u})+\epsilon_k(s,\mathbf{u})),
		\label{eq:w_k}
	\end{equation}
	where we have:
	\begin{equation}
		\begin{aligned}
			E_k(s,\mathbf{u})=&\frac{d^{\bpi_k}(s,\mathbf{u})}{\mu(s,\mathbf{u})}|Q_k-\mathcal{B}^*Q_{k-1}| \\
			&\cdot \exp(-|Q_k-Q^*|)\left(1+\sum_{i=1}^{n}\prod_{\substack{j=1 \\ j \neq i}}^{n}\pi^j_k-n\prod_{i=1}^{n}\pi^i_k\right),
		\end{aligned}
		\label{eq:e_k}
	\end{equation}
	where $ Z^* $ is the normalization factor, and $\epsilon_k(s,\mathbf{u})$ is a negligible term when the probability of reversing back to the visited state is small or the number of steps agents take to revisit a previous state is large.
	\label{theo:weight}
\end{theorem}

\begin{proof}[Proof (Sketch)]
	We give a sketch of the steps involved for completeness below. The complete proof is provided in the Appendix~\ref{proof:theo_1}. The derivation of optimal weights consists of the following major steps: (i) Use a relaxation and Jensen's inequality to obtain a more tractable upper bound of the regret objective for minimization. (ii) Formulate the Lagrangian for the new optimization problem and analyze its KKT conditions. (iii) Compute various terms in the KKT condition and, in particular, analyze the gradient of $ Q_k $ with respect to weights $ p_k $ (defined through the weighted Bellman equation) by leveraging the implicit function theorem (IFT). (iv) Derive the optimal projection weights in closed form by setting the Lagrangian gradient to zero and applying KKT and its slackness conditions.
	
	{\em Step 1: Relaxing the objective and using Jensen's Inequality.} To begin with, we replace the original optimization objective function, the policy regret, with a relaxed upper bound. This replacement can be achieved through the following inequality:
	\begin{equation}
		\eta(\pi^*)-\eta(\pi_k) \le \mathbb{E}_{d^{\bpi_k}(s,\mathbf{u})}[|Q_k-Q^*|(s,\mathbf{u})].
		\label{eq:obj_2}
	\end{equation}
	The proof of this result is given in the appendix. The key idea is to rewrite the regret using the expectation of the action-value functions with respect to discounted state distribution $ d^{\pi_k} $. After that, we adopt Jensen's inequality~\cite{mcshane1937jensen} to continue relaxing the intermediate objective function. 
%
	Consider a convex function $ g(x) = \exp(-x) $, a new optimization objective relaxed via Jensen's inequality generated from Equation~\eqref{eq:obj_2} becomes:
	\begin{equation}
		\min_{w_k} \quad -\log\mathbb{E}_{d^{\bpi_k}(s,\mathbf{u})}[\exp(-|Q_k-Q^*|)(s,\mathbf{u})],
		\label{eq:obj_3}
	\end{equation}
	where the constraints still hold for the new optimization objective.
	
	{\em Step 2: Computing the Lagrangian.} In this step, we leverage the Lagrangian multiplier method to solve the new optimization problem in Equation~\eqref{eq:obj_3}. For simplicity, we use $ p_k $ that absorbs the data distribution $ \mu $ into $ w_k $. The constructed Lagrangian is:
	\begin{equation*}
		\begin{aligned}
			\mathcal{L}(p_k;\lambda,\psi)=&-\log\mathbb{E}_{d^{\bpi_k}(s,\mathbf{u})}[\exp(-|Q_k-Q^*|)(s,\mathbf{u})] \\
			&+\lambda(\sum_{s,\mathbf{u}}p_k-1)-\psi\T p_k,
		\end{aligned}
	\end{equation*}
	where $ p_k $ is the weight $ w_k $ multiplied by the data distribution $ \mu $, and $ \lambda,\psi $ are the Lagrange multipliers.
	
	{\em Step 3: Computing the Gradients Required in the Lagrangian.} According to the first constraint in Equation~\eqref{eq:obj_1}, the gradient $ \frac{\p Q_k}{\p p_k} $ can be computed via IFT given by:
	\begin{equation*}
		\frac{\p Q_k}{\p p_k} = -[\mathrm{diag}(p_k)]^{-1}[\mathrm{diag}|Q_k-\mathcal{B}^*Q_{k-1}|].
	\end{equation*}
	
	We also derive the gradient $ \frac{\p d^{\bpi_k}(s,\mathbf{u})}{\p p_k} $ for solving the Lagrangian. The derivation details are given in the appendix.
	
	{\em Step 4: Deriving the Optimal Weight.} After having the equation for two gradients and an expression of the Lagrangian, we can compute the optimal $ p_k $ via an application of the KKT conditions, which needs to set the partial derivative of the Lagrangian equaling to zero, as:
	\begin{equation*}
		\frac{\p \mathcal{L}(p_k;\lambda,\psi)}{\p p_k}=0,
	\end{equation*}
	where the optimal weight $ w_k $ can be acquired from the $ p_k $.
	
\end{proof}

The theoretical results shed light on the key factors determining an optimal sampling weight for experience replay. Specifically, the optimal weights consist of four components relating to the Bellman error, the value enhancement, the joint action probability, and the on-policiness of available transitions. We will interpret these four components next, provide the analyses of some special cases in which the transitions will be assigned with higher weights, and develop a deep MARL algorithm through approximations of the optimal sampling weights.

{\em Bellman error $ |Q_k-\mathcal{B}^*Q_{k-1}| $}: is the estimation of the action value function after the Bellman update. This term measures the distance between the estimation and the Bellman target. A significant difference in this term means higher hindsight Bellman error and will lead to higher sampling weight assignment. This character is also similar to the prioritization criterion used in PER, which nevertheless considers more about the Bellman error in the previous iterations, i.e., $ |Q_{k-1}-\mathcal{B}^*Q_{k-2}| $.

{\em Value enhancement $ \exp(-|Q_k-Q^*|) $}: As we compute the absolute value between the current and optimal action-value function, the value enhancement term indicates that any transitions with less accurate action values compared to the optimal value estimation (i.e., a wider gap between $ Q_k $ and $ Q^* $) after the Bellman update should be assigned with lower weights. Conversely, a high sampling weight will be given if the current action value is approaching the optimal one.

{\em Joint action probability $ 1+\sum_{i=1}^{n}\prod_{j=1,j \neq i}^{n}\pi^j_k-n\prod_{i=1}^{n}\pi^i_k $}: The agent policies determine the probabilities of choosing certain actions. This result turns out that the optimal sampling weights depend on the individual policy of each agent as well, which is unique in the MARL task. According to this term, higher sampling weights will be assigned to transitions only if one agent's action probability is small in the transition while all other agents' action probabilities are large. This is a little counter-intuitive because we give higher weights to transitions with more differentiated action probabilities rather than similar ones. We will provide a thorough analysis in section~\ref{subsec:case} regarding studying the condition for the highest weight assignment in the general multi-agent scenario.

{\em Measurement of on-policy transitions $ \frac{d^{\bpi_k}(s,\mathbf{u})}{\mu(s,\mathbf{u})} $: }	
The efficient update of the joint action value function can be achieved by focusing on transitions that are more possibly to be visited by the current policy, i.e., with a higher $d^{\bpi_k}(s,\mathbf{u})$. Such strategy has been empirically studied in existing works~\cite{sinha2022experience}. Adding this term can speed up the search for the optimal $ Q_k $ close to $ Q^* $.

\subsection{Approximated Weights via Joint Action Probability Studies}
\label{subsec:case}

Theorem~\ref{theo:weight} shows terms determining the sampling weights needed for transitions, where a function of the joint action probability is the new result for MARL tasks. Although numerical calculation for the joint action probability is available, to lower the computational complexity when the environment has many agents involved, we develop an approximated weighting scheme that can determine the joint action probability via action probabilities and action-value functions of agents. For this purpose, we present a new theorem indicating the condition for obtaining maximum probability and several special case studies.

For the environment, we consider a general MARL scenario with agent space of $ n $, where we have $ a \in A \equiv \{1,\dots,n\} $. Every step, each agent $ a $ selects an action from its action space $ U^a $, following $ u_a^i \in U^a \equiv \{u_a^1,...,u_a^{m_a}\} $, where $ m_a $ is the size of action space of agent $ a $. Let the $ \bar{u}_a \in U^a $ denote the selected action of agent $ a $ at the step $ k $ from the action space. Due to the CTDE manner of MARL algorithms, the joint action value function space $ \mathcal{Q} $ contains the combinations of $ u_a^{i} $ ($ i $ ranges from 1 to $ m_a $) for each agent $ a $. For simplicity, we use $ Q_i $ to be the shorthand of $ Q(s, u_1^{i_1}, \dots, u_n^{i_n}) $, which represent a random action value function from $ \mathcal{Q} $ space. In particular, considering one selected action combination $ \mathbf{\bar{u}}=(\bar{u}_1,\dots,\bar{u}_n) $, the joint action-value function is $ \bar{Q} = Q(s, \mathbf{\bar{u}}) $. Since we use Boltzmann policy to compute the action probability, for one agent $ a \in A $ with the action $ \bar{u}_a $, its individual policy is:
\begin{equation}
	\begin{aligned}
		\pi_k^{a}=\frac{e^{\mathbb{E}_{u_{-a}^{j}\sim\mu}Q(s,\bar{u}_a,u_{-a}^{j})}}{\sum_{i=1}^{m_{a}} e^{\mathbb{E}_{u_{-a}^{j}\sim\mu}Q(s,u_{a}^i,u_{-a}^{j})}} 
		=\frac{e^{\sum_{-a}\sum_{j}\mu Q(s,\bar{u}_a,u_{-a}^{j})}}{\sum_{i=1}^{m_a} e^{\sum_{-a}\sum_{j}\mu Q(s,u_{a}^i,u_{-a}^{j})}},
	\end{aligned}
	\label{eq:pi_a}
\end{equation}
where $ -a $ represents all the agents except for target agent $ a $, and $ \mu $ is short for $ \mu(s,u_1^{i_1},\dots,u_n^{i_n}) $ representing the data distribution.


Under the general environmental setting, the state $ s $ will be fixed at each iteration, and the size of the action value function space $ \mathcal{Q} $ is $ \prod_{a=1}^{n}m_a $ with the dimension of $ n $. Let the following function denotes the joint action probability:
\begin{equation}
	f \eqdef 1+\sum_{i=1}^{n}\prod_{\substack{j=1 \\ j \neq i}}^{n}\pi^j_k-n\prod_{i=1}^{n}\pi^i_k,
	\label{eq:joint_prob}
\end{equation}
and we will provide another theorem indicating the conditions where we can acquire the maximum value of joint probability $ f $ in Equation~\eqref{eq:e_k}.

\begin{theorem}[Maximum probability conditions]
	Considering a selected action value $ \bar{Q} $ with action combination $ \mathbf{\bar{u}} $ of the step $ k $, the joint action probability function reaches its maximum $ f_{\rm max} $ if and only if the value of each action probability $ \pi_k^i $ is on the boundary (i.e., either 0 or 1) as well as at least one probability $ \pi_k^a $ equals to 0.
	\label{theo:cond}
\end{theorem}

\begin{proof}
	See Appendix~\ref{proof:theo_2}.
\end{proof}

Based on Theorem~\ref{theo:cond}, a higher joint action probability will be assigned to $ \bar{Q} $ of which agents' action probabilities are on the boundary of the interval $ [0,1] $ and at least one of the agents have its probability equaling $ 0 $. This conclusion casts light on determining the approximated sampling weights for MARL tasks. To better illustrate such an idea, we introduce several special case studies with respect to the selected $ \bar{Q} $ at step $ k $.

{\em Case 1: single large value $ \bar{Q} $.} In this case, we assume only one action value $ \bar{Q} $ out of the action value function space $ \mathcal{Q} $ is large, and values $ Q_i $ of other action combinations elsewhere are negligibly small, represented by $ \nu $. These small values obey $ \nu \approx 0 $ and $ \nu \ll Q_m $. Therefore, according to Equation~\eqref{eq:pi_a} and Theorem~\ref{theo:cond}, the joint action probability for the selected action combination of $ \bar{Q} $ is lower since the action probability $ \pi_k^a $ for each agent is similarly large. In contrast, the action combinations with only one action difference (e.g., $ Q(s,\mathbf{\bar{u}}_{-a_1},u_{a_1}^i) $) will be given with high weights $ \alpha_h $. The remaining position, such as the action combinations with two or more different actions, along with $ \bar{Q} $, will be assigned with low weight $ \alpha_l $.

{\em Case 2: dual large values $ \bar{Q},\bar{Q}' $.} We propose only two large values $ \bar{Q} \approx \bar{Q}' $ under this setting. Other positions are filled with negligible value $ \nu $. $ \bar{Q}' $ is the very same as $ \bar{Q} $ except that one agent's action is different, i.e., $ Q_m=Q(s,\mathbf{\bar{u}}) $ and $ \bar{Q}'=Q(s,\mathbf{\bar{u}}_{-a_1},u_{a_1}^i) $. Since two large values equally share the importance over the action value space, based on given equations/conditions, we can drive those action combinations with only one different agent's action other than agent $ a $ will be assigned with medium weights $ \alpha_m $, e.g., $ Q(s,\mathbf{\bar{u}}_{-(a_1,a_2)},u_{a_1}^i,u_{a_2}^j) $. The positions where action combination with one action difference over agent $ a_1 $ will receive high weights $ \alpha_h $, e.g., $ Q(s,\mathbf{\bar{u}}_{-a_1},u_{a_1}^{i'}) $. Besides, we will give low weight $ \alpha_l $ for other locations.

{\em Case 3: isolated large value $ \bar{Q}'' $ and $ \bar{Q} $.} Apart from given $ \bar{Q} $, we assume an isolated large value with two or more actions different from $ \bar{Q} $, i.e., $ \bar{Q}''=Q(s,\mathbf{\bar{u}}_{-(a_1,a_2)},u_{a_1}^i,u_{a_2}^j) $, which exists somewhere in the action value function space, and satisfies $ \bar{Q} \approx \bar{Q}'' $. Other action values are $ \nu $. In this situation, both two large values share the same importance, and we will assign the medium weight $ \alpha_m $ to $ \bar{Q} $, $ \bar{Q}'' $, and the action combinations having one action different over agent $ a_1 $ or $ a_2 $, such as $ Q(s,\mathbf{\bar{u}}_{-(a_1,a_2)},u_{a_1}^{i'},u_{a_2}^j) $. The rest of the action combination values will be allocated with low weight $ \alpha_l $. This special case demonstrates that if one or more action values are extraordinarily large, indicating another joint policy candidate with latent high joint action probability, we should also heed such equivalent competitor and its local search.

We can establish the approximation structure by studying from mentioned special cases. The scaled weights $ \alpha_l $, $ \alpha_m $, and $ \alpha_h $ provide an alternative solution that spares us from directly using the numerically computed sampling weights to solve the latent computational cost, yet the performance remains mainly impervious.


\subsection{Proposed Algorithms}

Our analytical results in Theorem~\ref{theo:weight} identify four key factors determining the optimal projection weights. The first term, relating to the Bellman error, recovers the designs in classic prioritized experience replay. Specifically, when the Bellman error of a particular transition is high, which indicates a wide hindsight gap between $ Q_k $ and the Bellman target, we may consider assigning a larger weight to this transition. Besides, the value enhancement term selectively emphasizes the importance of incoming transitions. Based on the difference between current $ Q_k $ and ideal $ Q^* $, it will compensate the near-optimal $ Q_k $ with larger importance while penalizing non-optimal $ Q_k $ with a smaller weighting modifier. Moreover, similar to previous studies, the measurement of on-policy transitions in the weighting expression underlines the useful information carried by more current, on-policy transitions. 

Our analysis also identifies a new term reflecting the interaction among agents in the MARL scenario: the joint action probability of multiple agents, which is crucial in obtaining optimal sampling weights for specific transitions. We interpret the joint action probability term in optimal weights constrained by the given condition: one agent's action probability is small in the transition, while all other agents' action probabilities are large. We increase the weights for transitions satisfying this condition. On the contrary, we decrease the weight if the condition fails to be satisfied.

Following these theoretical results, we propose a MARL algorithm for collective priority optimization, \name, with regret-minimizing joint policy in multi-agent environments. We consider a new loss function with respect to the optimal sampling weights $ w_k $ applied to the Bellman equation of, i.e.,
\begin{equation}
	L_{\rm \name} = \sum_{i=1}^{b}w_k(s,\mathbf{u})(Q_k-y_i)^2(s,\mathbf{u}),
\end{equation}
where $ b $ is the batch size, and $ y_i=\mathcal{B}^*Q_{k-1} $ is a fixed target that can be obtained through a target network.

The Bellman error and joint action probability of all agents in the environment can be directly computed using Theorem~\ref{theo:weight}. To compute the sampling weights for value enhancement term in practice, we use the backbone of the classical value factorization MARL algorithm QMIX and leverage the unrestricted joint action-value function $ Q^* $ to compute the approximated optimal action-value function quantitatively. Ideally, we could have included measurement of on-policy transitions term in the computation, but it is not readily available since the distribution $ d^{\bpi_k}(s,\bu) $ in the numerator cannot be directly obtained. It is also worth mentioning that such term can be dismissed and the other terms in the weight expression are enough to provide a good estimate and lead to performance improvements, as shown in existing work~\cite{kumar2020discor}. 
Furthermore, based on our previous discussion, we designed an approximated counterpart, \name Approximation, by setting the threshold and scaling sampling weights values into low, medium, and high ones. The pseudo-codes are provided in Appendix~\ref{subsec:alg}.

\section{Experiments}
In this section, we present our experimental results on Predator-Prey and SMAC benchmarks and demonstrate the effectiveness of \name by comparing the results with several state-of-the-art MARL baselines. Additionally, we compare \name and \name Approximation with other experience replay methods adapted from single-agent RL to multi-agent environments. Each comparison is implemented independently with fixed and optimized~\cite{mei2022bayesian} hyperparameters. We also conduct the ablation experiments to discuss the contribution of each term mentioned in Theorem~\ref{theo:weight}. More implementation details are provided in Appendix~\ref{subsec:setup}. The code
has been made available at: \url{https://github.com/ysmei97/MAC-PO}.

\begin{figure*}[ht!]
	\centering
	\begin{subfigure}[ht]{0.33\textwidth}
		\centering
		\includegraphics[width=\textwidth]{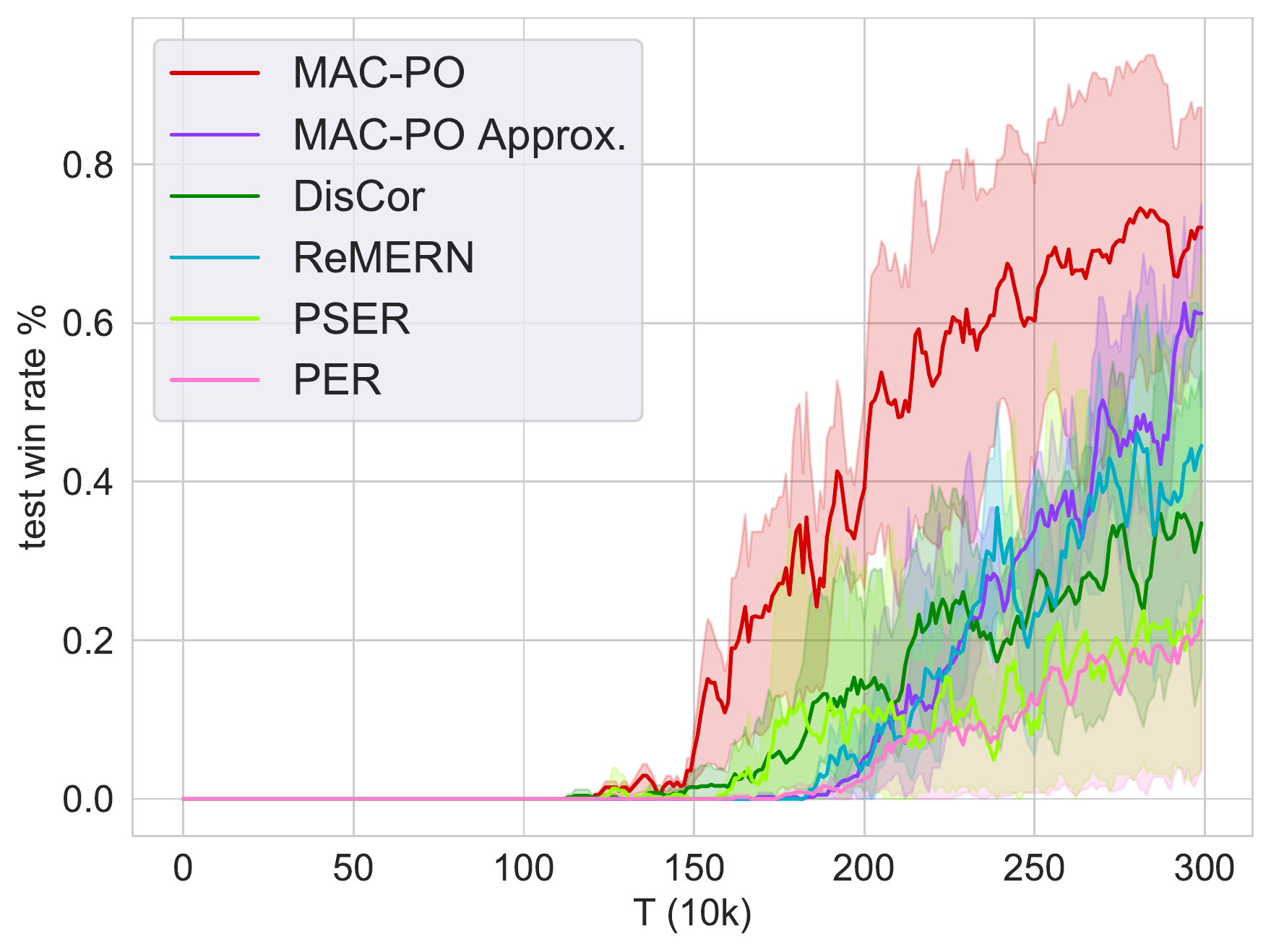}
		\caption{\small 3s\_vs\_5z (hard)}
	\end{subfigure}
	\begin{subfigure}[ht]{0.33\textwidth}
		\centering
		\includegraphics[width=\textwidth]{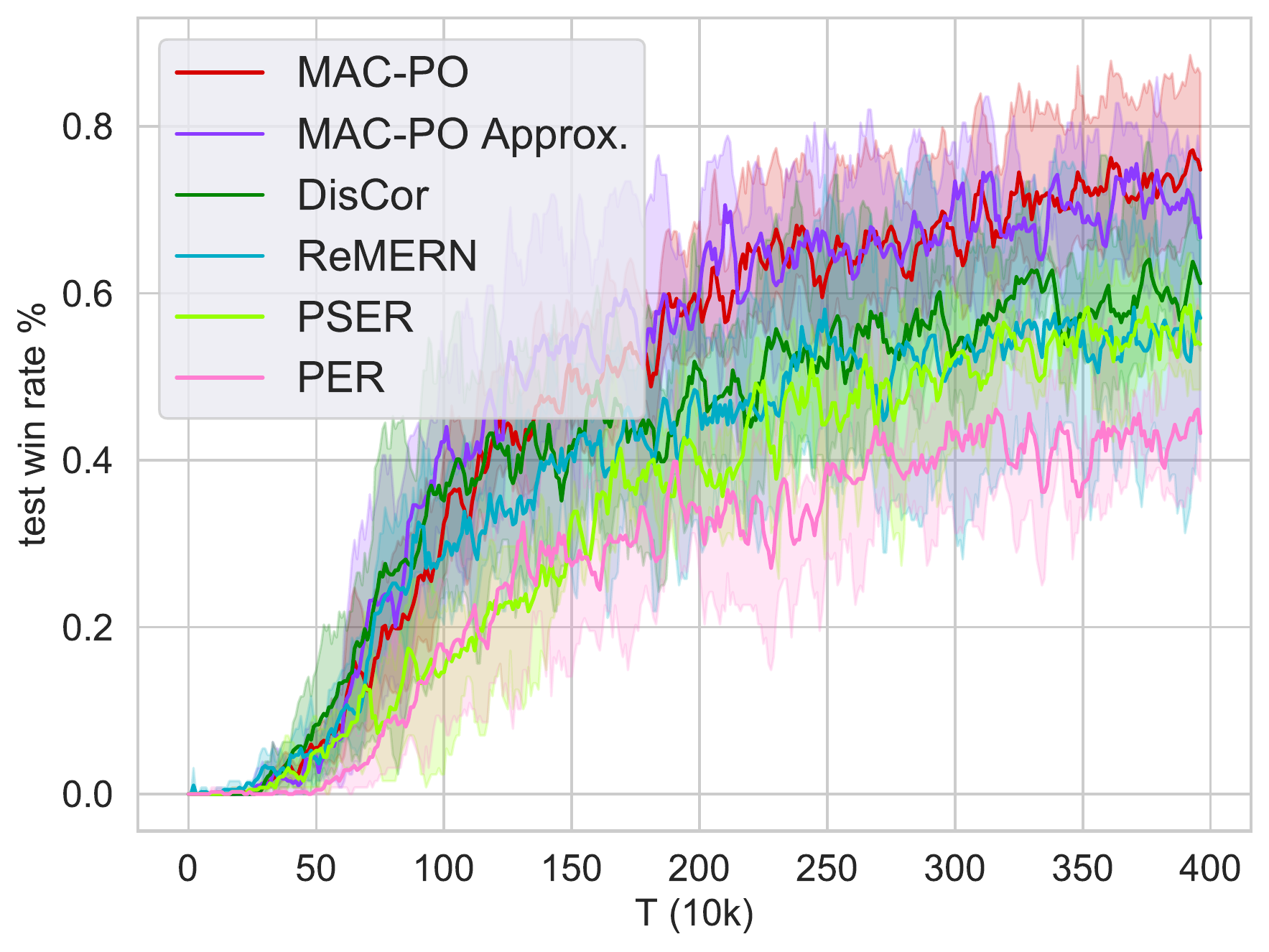}
		\caption{\small 5m\_vs\_6m (super hard)}
	\end{subfigure}
	\begin{subfigure}[ht]{0.33\textwidth}
		\centering
		\includegraphics[width=\textwidth]{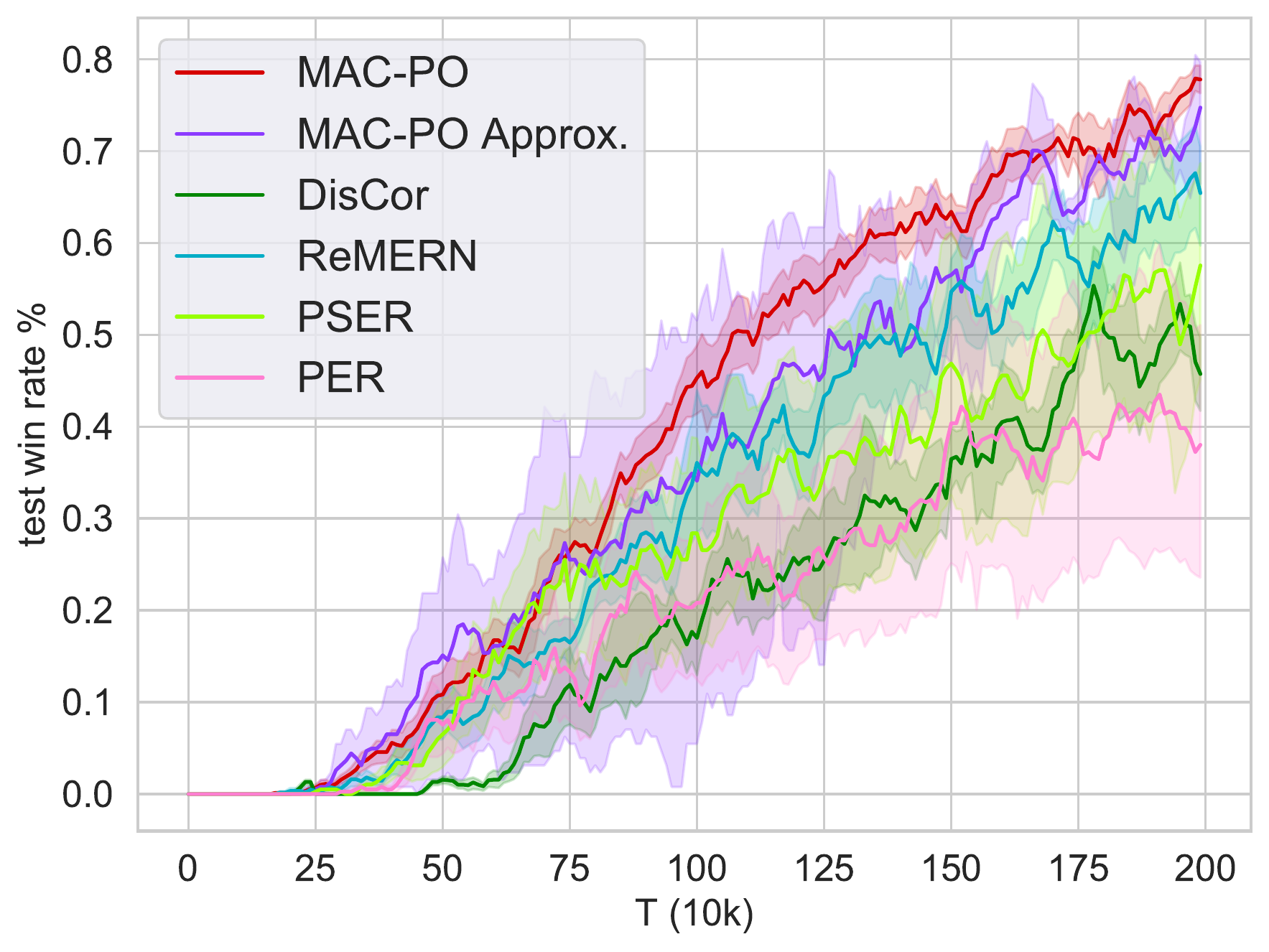}
		\caption{\small MMM2 (super hard)}
	\end{subfigure}
	\caption{Comparison between \name, \name Approximation and other experience replay methods on three SMAC maps (from hard to super hard), where \name outperforms the second best one -- \name Approximation by 10\%, 6\%, and 4\% on each map, respectively.}
	\label{fig:cmp}
\end{figure*}

\begin{figure*}[ht!]
	\centering
	\begin{subfigure}[ht]{0.48\textwidth}
		\centering
		\includegraphics[width=\textwidth]{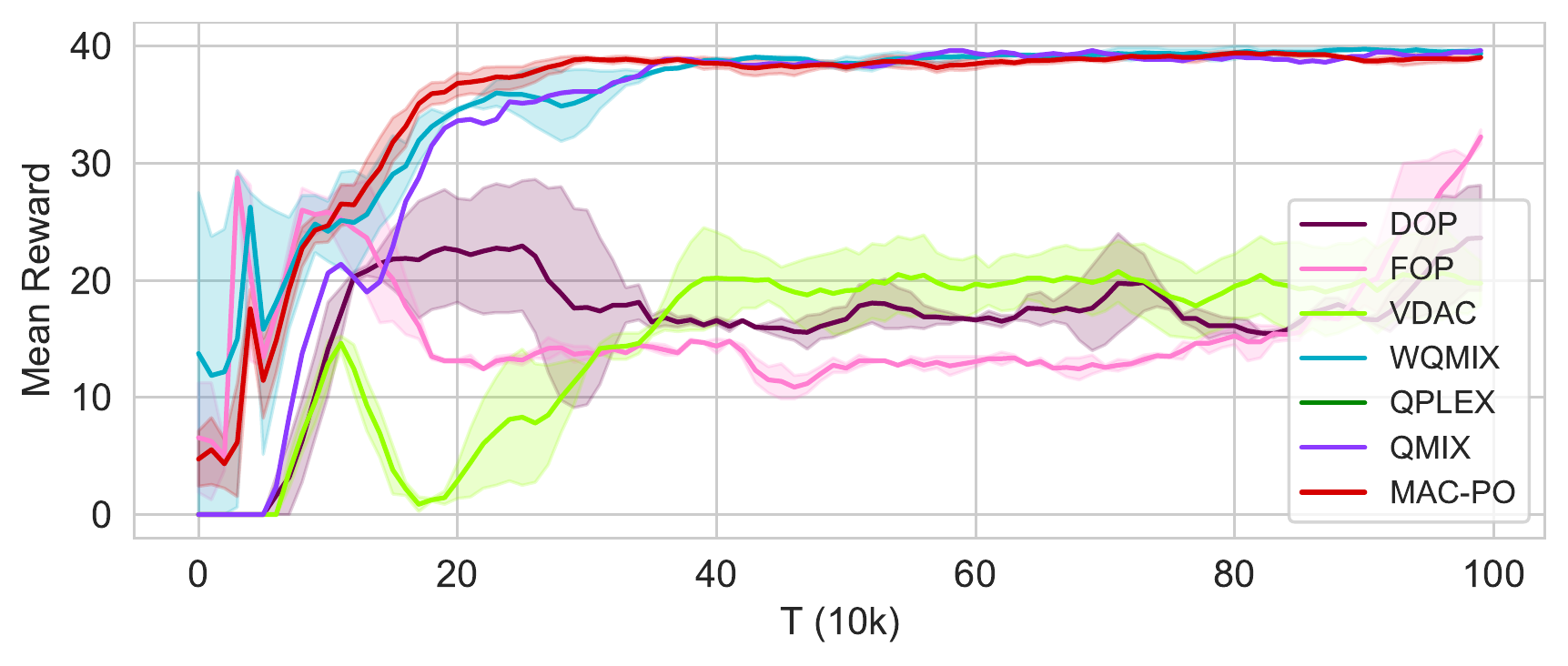}
		\caption{\small No punishment}
		\label{fig:pp_0}
	\end{subfigure}
	\begin{subfigure}[ht]{0.50\textwidth}
		\centering
		\includegraphics[width=\textwidth]{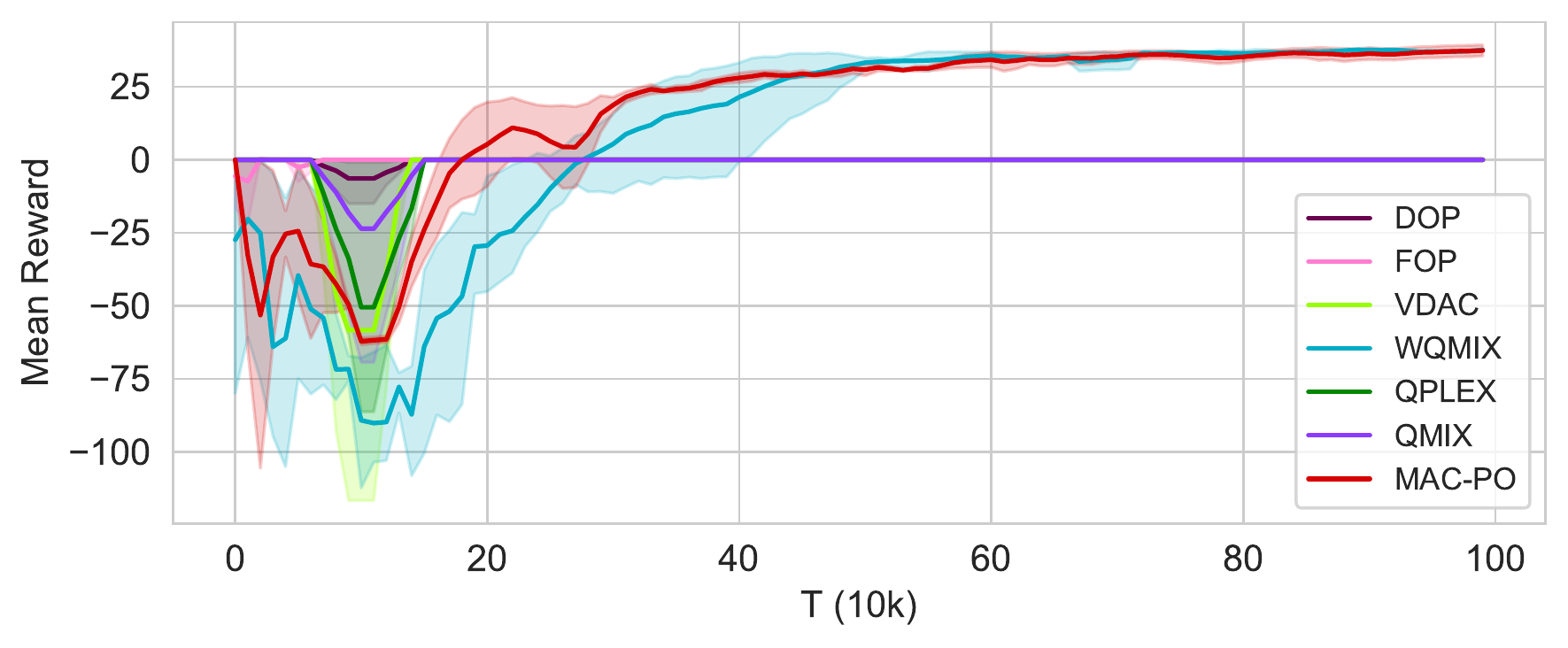}
		\caption{\small Punishment $ = - $1.5}
		\label{fig:pp_1.5}
	\end{subfigure}
	\caption{Average reward per episode on the Predator-Prey tasks for \name and other MARL algorithms of two settings, where \name shows the profoundly better convergence speed and pretty good results.}
	\label{fig:pp}
\end{figure*}

\subsection{Comparison with Existing Experience Replay Methods}

In this experiment, we compare \name with other experience replay methods in the multi-agent environment SMAC. Since existing experience replay methods are designed for the single-agent scenario, we borrow their core designs and transplant them to multi-agent environments by considering all agents as the conceptual agent to match the single-agent target in their original settings. Such transplanting will recover the most important ingredients from RL to MARL to the greatest extent. It is worth mentioning that many other algorithms are also introducing a variety of experience replay schemes. Some of them~\cite{oh2021model, sujit2022prioritizing} depend on new components, and others~\cite{saglam2022actor} have different algorithm architectures. Since the backbone MARL algorithm of our choice in this experiment is QMIX, we do not expect a significant change over the algorithm architecture (e.g., actor-network) or major components (e.g., loss structure) as presented in other approaches to realize a relatively fair comparison.

The first approach for comparing is PER~\cite{schaul2016prioritized}. Due to the equivalence between loss functions and non-uniform sampling for experience replay~\cite{fujimoto2020equivalence}, we reconstruct PER scheme by computing weights only related to the current TD error regarding the joint action-value function. We also compare our method with the one mentioned in DisCor~\cite{kumar2020discor}, where the weights are calculated from the production of Bellman error and value enhancement terms, and ReMERN~\cite{liu2021regret}, which has an additional term describing action likelihood, and we extend it to the multi-agent case. Besides, we transplanted the mechanism from PSER~\cite{brittain2019prioritized}, which is another extension of PER, by applying an additional decay factor and window size on the weights for coming transitions. For this experiment, we set the decay factor as 0.4 and the window size as 5.

Figure~\ref{fig:cmp} shows the performance comparison among \name and other experience replay algorithms on three maps of SMAC benchmark, which are \textit{3s\_vs\_5z}, \textit{5m\_vs\_6m}, and \textit{MMM2}. Compared to \name, other experience replay schemes underperform in improving learning performance. DisCor and ReMERN have higher final winning rates than the regular PER, demonstrating the effectiveness of additional terms. PSER also acts better than PER owing to its decaying mechanism for selecting history transitions. All transplanted algorithms from single-agent scenario act unstably in the multi-agent environment, as we can notice the variance reflected by the shaded area in Figure~\ref{fig:cmp}.

Besides, we also test our approximated sampling weight approach, shown as \name Approximation in Figure~\ref{fig:cmp}. We set the higher weight $ \alpha_h $ as 0.75, medium weight $ \alpha_m $ as 0.5, and lower weight $ \alpha_l $ as 0.25. The final result is almost identical to the original \name with small nuance. For the original \name, the computational complexity of obtaining sampling weights will increase if more agents get involved. Since the approximated \name uses scaled weights instead of numerical results, it will improve the computational efficiency of the original \name at the price of slightly sacrificing the overall performance.

\begin{figure*}[ht!]
	\centering
	\begin{subfigure}[t]{0.33\textwidth}
		\centering
		\includegraphics[width=\textwidth]{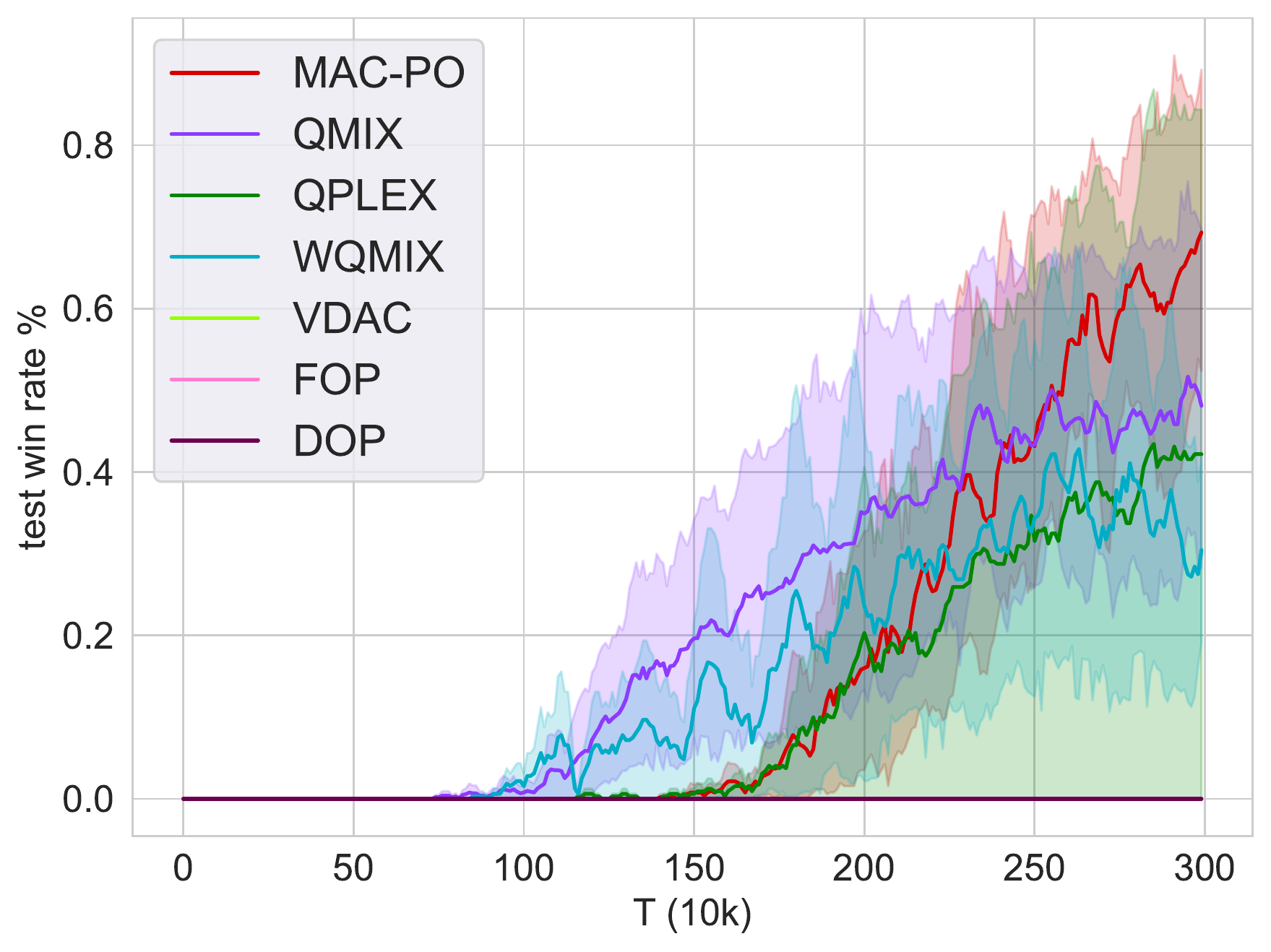}
		\caption{\small 3s\_vs\_5z (hard)}
		\label{fig:3svs5z}
	\end{subfigure}
	\begin{subfigure}[t]{0.33\textwidth}
		\centering
		\includegraphics[width=\textwidth]{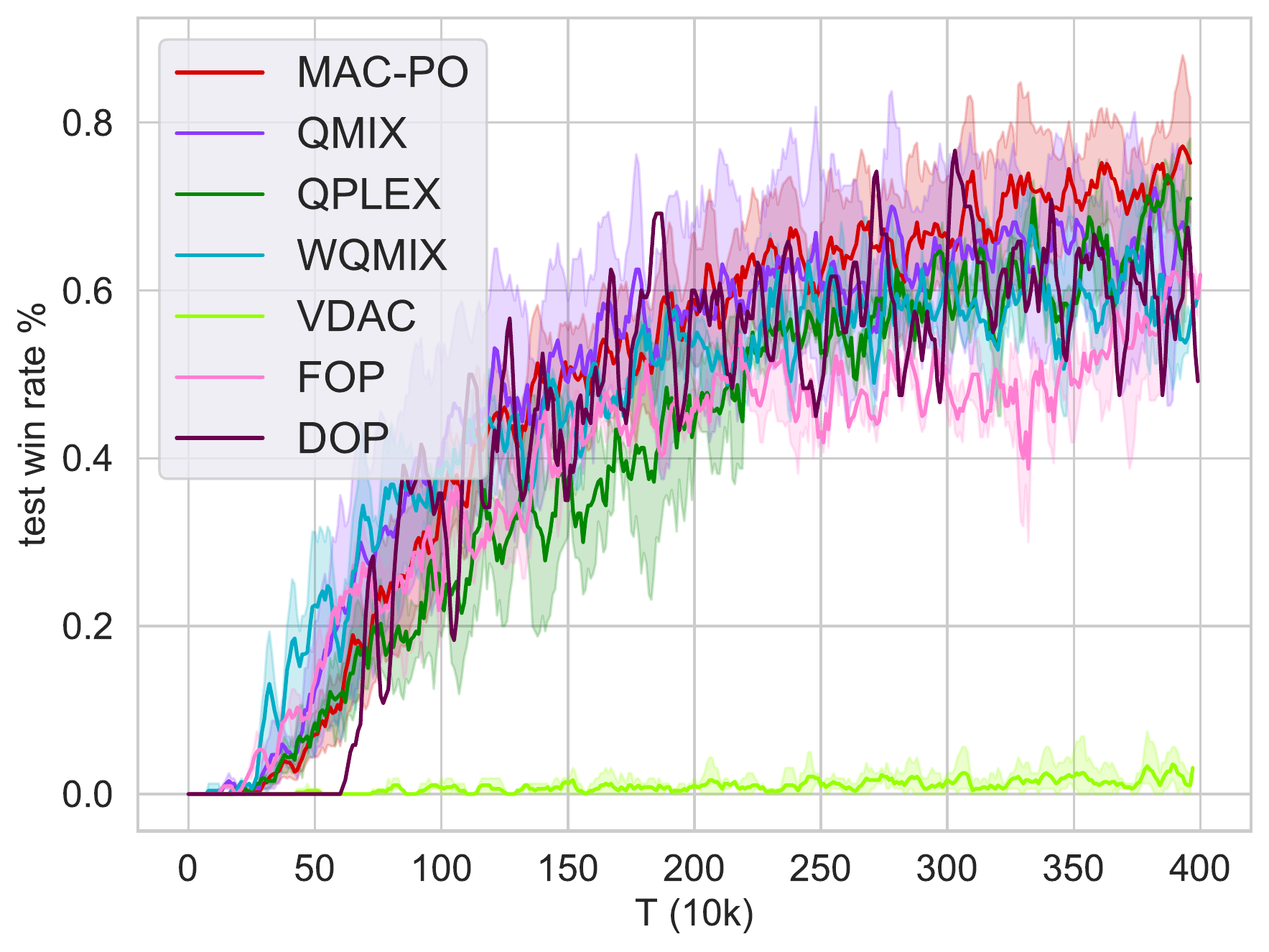}
		\caption{\small 5m\_vs\_6m (hard)}
		\label{fig:5mvs6m}
	\end{subfigure}
	\begin{subfigure}[t]{0.33\textwidth}
		\centering
		\includegraphics[width=\textwidth]{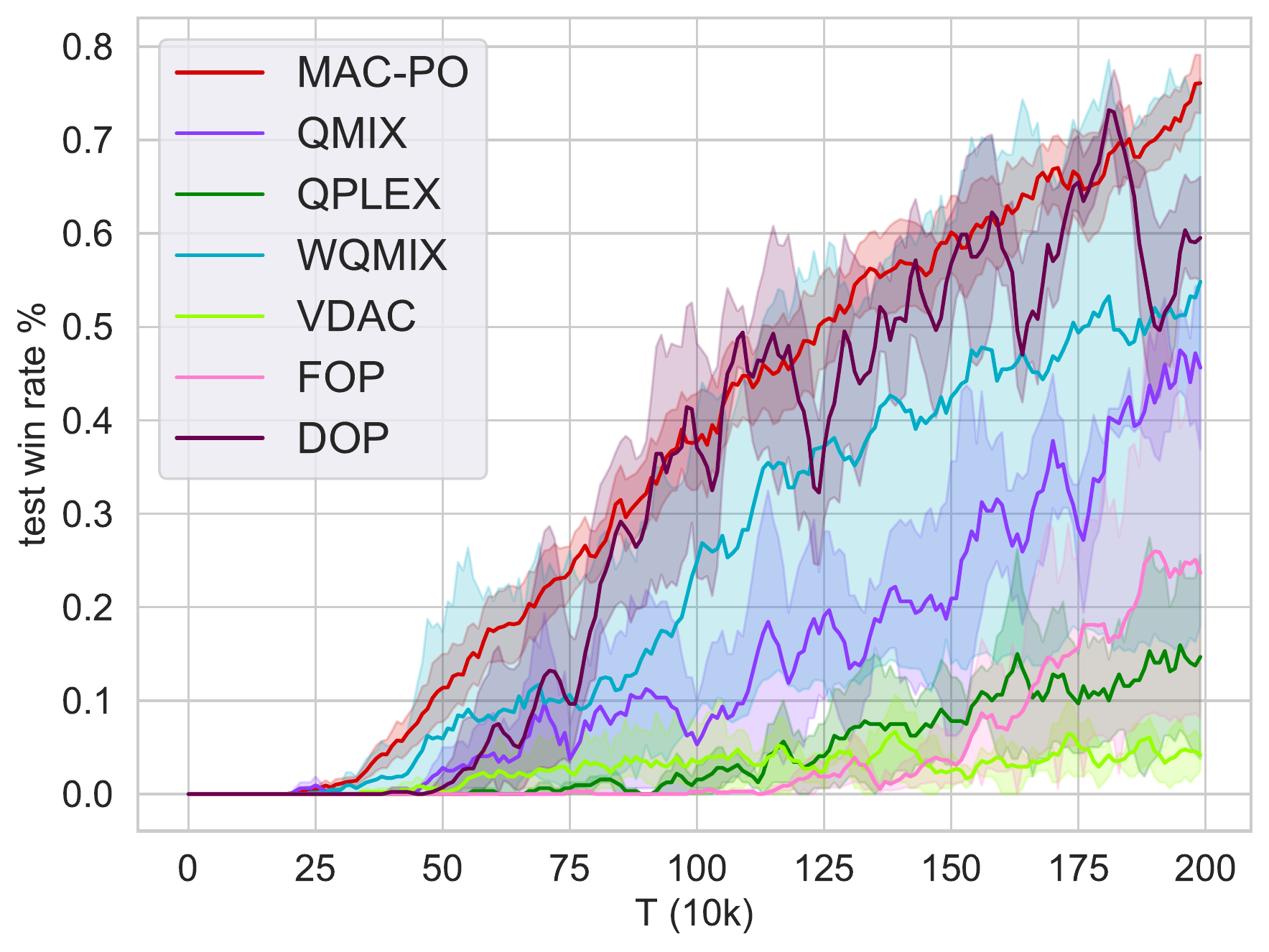}
		\caption{\small MMM2 (super hard)}
		\label{fig:mmm2}
	\end{subfigure}
	\caption{Comparison between \name and other MARL algorithms on three SMAC maps (from hard to super hard). \name achieves the best results with the optimal weighting scheme alone, outperforming the second best result by 11\%, 4\%, and 16\% on each map, respectively. }
	\label{fig:smac}
\end{figure*}

\subsection{Comparison with MARL Algorithms}

\subsubsection{Predator-Prey}

We compare \name with MARL algorithms on a complex partially-observable multi-agent cooperative environment, Predator-Prey, that involves eight agents in cooperation as predators to catch eight prey on a 10$ \times $10 grid. In this task, a successful capture with the positive reward of 1 must include two or more predator agents surrounding and catching the same prey simultaneously, requiring a high level of cooperation. A failed coordination between agents to capture the prey, which happens when only one predator catches the prey, will receive a negative punishment reward. We select multiple state-of-the-art MARL algorithms for comparison, which include value-based factorization MARL algorithm (i.e., QMIX, WQMIX~\cite{rashid2020weighted}, and QPLEX), decomposed policy gradient method (i.e., VDAC~\cite{su2021value}), and decomposed actor-critic approaches (i.e., FOP~\cite{zhang2021fop} and DOP~\cite{wang2020dop}). All mentioned baselines have shown strength in handling MARL tasks in existing works.

Figure~\ref{fig:pp} shows the performance of seven algorithms with different punishments, where all results show the effectiveness of \name. Besides, regarding efficiency, we can spot that \name has the fastest convergence speed in seeking the best policy. In Figures~\ref{fig:pp_1.5}, \name significantly outperforms other state-of-the-art algorithms in a hard setting requiring a higher level of coordination among agents as learning the best policy. Most MARL algorithms learn a sub-optimal policy where agents learn to work together with limited coordination. Although the performance of \name and WQMIX are similar, compared to the latter, \name converges to the optimal policy profoundly faster, demonstrating that our multi-agent optimal weighting scheme can efficiently learn from specific existing transitions.

\subsubsection{SMAC}
Next, we evaluate \name on the SMAC benchmark. We report the experiments on three maps consisting of two hard maps and one super-hard map. The selected baselines for this experiment are consistent with those in the Predator-Prey environment. The empirical results are provided in Figure~\ref{fig:smac}, demonstrating that \name can effectively generate optimal weight transitions on SMAC for achieving a higher win rate, especially when the environment becomes substantially complicated and harder, such as \textit{MMM2}. We can see that several state-of-the-art algorithms are brittle when significant exploration is undergoing without finding optimal sampling weights.

Specifically, \name performs well on hard maps, such as \textit{3s\_vs\_5z}, the best policy found by our optimal weighting approach significantly outperforms the remaining baseline algorithms regarding winning rate. For super-hard map \textit{MMM2}, \name, along with QMIX, WQMIX, and QPLEX, can learn a better policy than VDAC, DOP, and FOP. We achieve the highest winning rate by adopting our algorithm on \textit{MMM2}, showing the superiority of the optimal weighting scheme in utilizing past transitions.

\begin{figure}[t]
	\centering
	\includegraphics[width=0.84\linewidth]{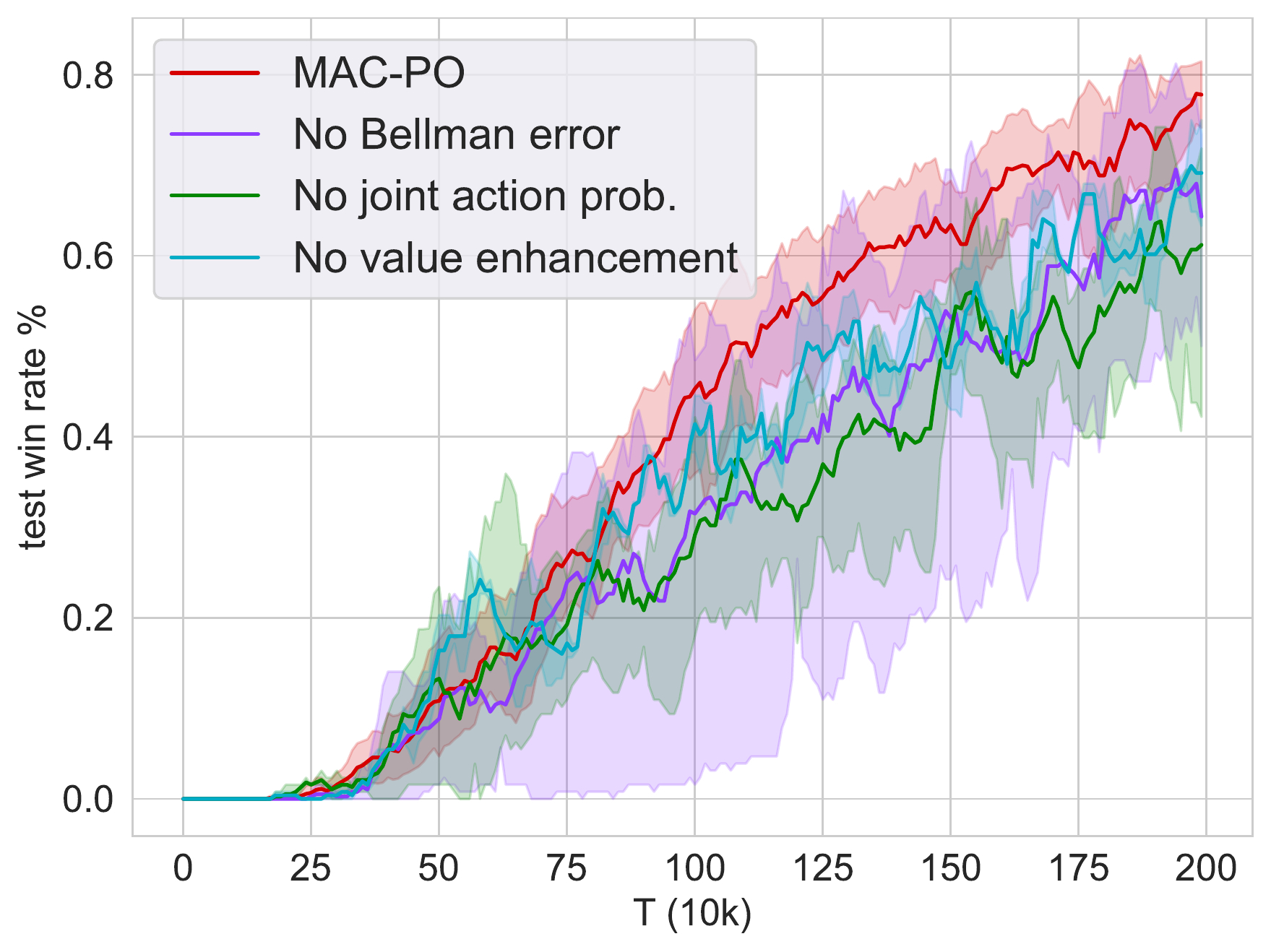}
	\caption{Ablations by disabling one term each for \name on MMM2 (super hard). The final winning rates decrease by 18\% for disabling the joint action probability term, 15\% for disabling the Bellman error term, and 10\% for disabling the value enhancement term.}
	\label{fig:ablation}
\end{figure}

\subsection{Ablation Experiments}

For ablations, we conduct experiments by disabling one term (mentioned in Theorem~\ref{theo:weight}) every trial to investigate their contribution to finding optimal sampling weights, respectively. The terms considered in these experiments are Bellman error, value enhancement, and joint action probability. Figure~\ref{fig:ablation} shows the results on \textit{MMM2}. Compared to the original result, missing any of the terms will be detrimental to the performance, and the tests without joint action probability have the lowest final winning rate, which is around 60\%. Such a phenomenon demonstrates that the interaction among agents is the critical factor in MARL tasks. The designing of the optimal weighting scheme without taking joint action probability into account will be less capable of achieving ideal results. Furthermore, the contributions of Bellman error and value enhancement terms are similar according to the given trend in Figure~\ref{fig:ablation}.

\section{Conclusion}

In this paper, we formulate multi-agent experience replay as a regret minimization problem and solve the optimal sampling weights in close form. The theoretical results illustrate key ingredients for an optimal experience replay in MARL settings. The results enable us to propose \name (with both exact and approximated weights) as a new MARL experience replay algorithm with optimized experience replay weights. Our experiment results in multiple MARL environments show the effectiveness of \name by demonstrating superior convergence and empirical performance over other experience replay solutions (adapted from single-agent RL) as well as state-of-the-art MARL methods.



\begin{acks}
	This research is based on work supported by the National Science Foundation under grant CCF-2114415 and partially by research gifts from CISCO and Meta. 
\end{acks}



\bibliographystyle{ACM-Reference-Format} 
\bibliography{reference}

\onecolumn
\appendix


\section{Nomenclature}
\label{subsec:notation}
Table~\ref{tab:notation} summarizes the common notations used in this paper.

\begin{table}[h]
	\centering
	\caption{Definitions of the common notations.}
	\label{tab:notation}
	\begin{tabular}{lc}
		\toprule
		Notation & Definition \\
		\midrule
		$ s $ & State of the environment \\
		$ a $ & Agent \\
		$ u $ & Agent's individual action \\
		$ \bu $ & Agents' joint action \\
		$ r $ & Reward \\
		$ \gamma $ & Discount factor \\
		$ \tau $ & Action-observation history \\
		$ \bpi $ & Joint policy \\
		$ \bpi^* $ & Expected optimal joint policy \\
		$ \eta(\pi) $ & Expected return under the joint policy $ \bpi $ \\
		$ d^{\bpi}(s) $ & Discounted state distribution \\
		$ Q(\cdot) $ & Action value function \\
		$ Q^*(\cdot) $ & Optimal action value function \\
		$ V(\cdot) $ & Value function \\
		$ V^*(\cdot) $ & Optimal value function \\
		$ A(\cdot) $ & Advantage function \\
		$ L(\cdot) $ & Loss function \\
		$ \mathcal{B}^*$ & Bellman operator: $ \mathcal{B}^*Q(s,\bu) \eqdef r(s,\bu) + \gamma\arg\max_{\bu'}\mathbb{E}_{s'}Q(s',\bu') $ \\
		$ w $ & Sampling weight \\
		$ \alpha $ & Scaled sampling weight \\
		\bottomrule
	\end{tabular}
\end{table}

\section{Proof of Theorem~\ref{theo:weight}}
\label{proof:theo_1}

We have provided the outline of the proof including four key steps. In this section, we present the detailed proof of the theorem. The optimization problem needed solving is:
\begin{equation*}
	\begin{aligned}
		\min_{w_k} \quad &\eta(\bpi^*)-\eta(\bpi_k) \\
		\textrm{s.t.} \quad &Q_k=\arg\min_{Q \in \mathcal{Q}} \mathbb{E}_{\mu}[w_k(s,\mathbf{u})(Q-\mathcal{B}^*Q_{k-1})^2(s,\mathbf{u})], \\
		&\mathbb{E}_{\mu}[w_k(s,\mathbf{u})]=1, \quad w_k(s,\mathbf{u}) \ge 0,
	\end{aligned}
\end{equation*}

This problem is equivalent to:
\begin{equation}
	\begin{aligned}
		\min_{p_k} \quad &\eta(\bpi^*)-\eta(\bpi_k) \\
		\textrm{s.t.} \quad &Q_k=\arg\min_{Q \in \mathcal{Q}} \mathbb{E}_{p_k}[(Q-\mathcal{B}^*Q_{k-1})^2(s,\bu)], \\
		&\sum_{s,\bu}p_k(s,\bu)=1, \quad p_k(s,\bu) \ge 0,
	\end{aligned}
	\label{eq:ap_obj_1}
\end{equation}
where $ p_k=w_k(s,\bu)\mu(s,\bu) $ is the solution to problem~\eqref{eq:ap_obj_1}.

To solve the optimization problem in Equation~\eqref{eq:ap_obj_1}, we needed to provide some definitions, which are \textit{total variation distance}, \textit{Wasserstein metric}, and \textit{the diameter of a set}.
\begin{definition}[Total variation distance]
	The total variation distance of the distribution P and Q is defined as $ D(P,Q)=\frac{1}{2} \Vert P-Q \Vert $.
\end{definition}
\begin{definition}[Wasserstein metric]
	For F,G two cumulative distribution function over the reals, the Wasserstein metric is defined as $ d_p(F,G) \eqdef \inf_{U,V} \Vert  U-V \Vert_p $, where the infimum is taken over all pairs of random variables (U,V) with cumulative distributions F and G, respectively.
\end{definition}
\begin{definition}[Diameter of a set]
	The diameter of a set A is defined as $ \mathrm{diam}(A)=\sup_{x,y \in A} m(x,y) $, where m is the metric on A.
\end{definition}

Furthermore, we introduce some mild assumption as follows:
\begin{assumption}
	The state space $ S $, action space $ U $ and observation space $ Z $ are compact metric spaces.
	\label{assum:metric}
\end{assumption}
\begin{assumption}
	The action-value and observation function are continuous on $ S \times U $ and $ Z $, respectively.
	\label{assum:continuous}
\end{assumption}
\begin{assumption}
	The transition function T is continuous regarding $ S \times U $ in the sense of Wasserstein metric: $ \lim_{(s,\bu)\rightarrow(s_0,\bu_0)} d_p(T(\cdot|s,\bu),T(\cdot|s_0,\bu_0)) $.
\end{assumption}
\begin{assumption}
	The joint policy $ \bpi $ is the product of each agent’s individual policy $ \pi^a $
	\label{assum:policy}
\end{assumption}
These assumptions can be satisfied in most MARL environments.


Let $ d^{\pi^a}(s) $ denote the discounted state distribution of agent $ a $, and $ d_i^{\pi^a}(s) $ denote the distribution where the state is visited by the agent for the $ i $-th time. Thus, we have:
\begin{equation}
	d^{\pi^a}(s) = \sum^{\infty}_{i=1} d^{\pi^a}_i(s),
\end{equation}
where each $ d^{\pi^a}_i(s) $ is given by:
\begin{equation}
	d^{\pi^a}_i(s)= (1-\gamma)\sum^\infty_{t_i=0}\gamma^{t_i}\Pr(s_{t_i}=s, s_{t_k}=s, \forall k=1,...,i-1),
\end{equation}
where the $ \Pr(s_{t_i}=s, s_{t_k}=s, \forall k=1,...,i-1) $ in this equation contains the probability of visiting state $ s $ for the $ i $-th time at $ t_i $ and a sequence of times $ t_k $, for $ k=1,..., i $, such that state $ s $ is visited at each $ t_k $. Thus, state $ s $ will be visited for $ i $ times at time $ t_i $ in total.

The following lemmas are proposed by Liu~\cite{liu2021regret}, where Lemma~\ref{lem:tvd} support the derivation of the Lemma~\ref{lem:d_pi}, and Lemma~\ref{lem:d_pi} demonstrates that $ \left| \frac{\p d^{\pi^a}(s)}{\p \pi^a(s)} \right| $ is a small quantity.
\begin{lemma}
	Let $ f $ be an Lebesgue integrable function. P and Q are two probability distributions, $ f \leq C $, then:
	\begin{equation}
		|\mathbb{E}_{P(x)}f(x)-\mathbb{E}_{Q(x)}f(x)| \le C\cdot D(P,Q).
	\end{equation}
	\label{lem:tvd}
\end{lemma}

\begin{lemma}
	Let $ \rho $ be the probability of the agent $ a $ starting from $ (s,u^a) $ and coming back to $ s $ at time step $ t $ under policy $ \pi^a $, i.e. $ \Pr(s_0=s,u^a_0=u^a,s_t=s,s_{1:t-1} \neq s; \pi^a) $, and $ \epsilon = \sup_{s,u^a}\sum_{t=1}^{\infty}\gamma^t\rho^{\pi^a}(s,u^a,t)$. We have:
	\begin{equation}
		\left| \frac{\p d^{\pi^a}(s)}{\p \pi^a(s)} \right| \le \epsilon d_1^{\pi^a}(s),
	\end{equation}
	where $ d^{\pi^a}_1(s)=(1-\gamma)\sum^\infty_{t_1=0}\gamma^{t_1}\Pr(s_{t_1}=s) $ and $ \epsilon \leq 1 $.
	\label{lem:d_pi}
\end{lemma}

%

Lemma~\ref{lem:tvd} and \ref{lem:d_pi} can be extended to suit the multi-agent scenario. Besides, we have the following lemma holds in MARL:
\begin{lemma}
	Given two policy $ \bpi $ and $ \tilde{\bpi} $, where $ \bpi=\frac{\exp(Q(s,\bu))}{\sum_{\bu'}\exp(Q(s,\bu'))} $ is defined as the Boltzmann policy, we have:
	\begin{equation}
		\mathbb{E}_{\bu\sim\tilde{\bpi}}[Q(s,\bu)]-\mathbb{E}_{\bu\sim\bpi}[Q(s,\bu)] \leq 1
	\end{equation}	
	\label{lem:inequal}
\end{lemma}

\begin{proof}
	Assume there are two joint actions $ \bu $ and $ \tilde{\bu} $. Let $ Q(s,\bu) = p $, $ Q(s,\tilde{\bu}) = q $ and let $ p \leq q $.
	\begin{equation*}
		\begin{aligned}
			\mathbb{E}_{\bu\sim\tilde{\bpi}}[Q(s,\bu)]-\mathbb{E}_{\bu\sim\bpi}[Q(s,\bu)] &\leq q - \frac{pe^p+qe^q}{e^p+e^q} \\
			&= q - \frac{p+qe^{q-p}}{1+e^{q-p}} \\
			&= q - p - \frac{(q-p)e^{q-p}}{1+e^{q-p}}.
		\end{aligned}
	\end{equation*}
	Let $ f(z)=z-\frac{ze^z}{1+e^z} $, the maximum point $ z_0 $ satisfies $ f'(z)=0 $, from which we further have $ 1 + e^{z_0} = z_0e^{z_0} $ where $ z_0\in(1,2) $. Therefore, we conclude:
	\begin{equation*}
		\mathbb{E}_{\bu\sim\tilde{\bpi}}[Q(s,\bu)]-\mathbb{E}_{\bu\sim\bpi}[Q(s,\bu)] \leq f(q-p) \leq z_0-1 \leq 1.
	\end{equation*}
\end{proof}

\begin{remark}
	The inequality in Lemma~\ref{lem:inequal} can be applied to both the situation where we have joint action of more than two agents and the situation regarding per-agent action.
\end{remark}

The following lemma is proposed by Kakade~\cite{kakade2002approximately}. It was originally proposed for the finite MDP, while it will also hold for the continuous scenario that is given by Assumption~\ref{assum:metric} and \ref{assum:continuous}.

\begin{lemma}
	For any policy $ \bpi $ and $ \tilde{\bpi} $, we have:
	\begin{equation}
		\eta(\tilde{\bpi})-\eta(\bpi)=\frac{1}{1-\gamma}\mathbb{E}_{d^{\tilde{\bpi}}(s,\bu)}[A^{\bpi}(s,\bu)],
	\end{equation}
	where $ A^{\bpi}(s,\bu) $ is the advantage function given by $ A^{\bpi}(s,\bu) = Q^{\bpi}(s,\bu)-V^{\bpi}(s) $.
	\label{lem:kakade}
\end{lemma}

\begin{lemma}
	Let $ \epsilon_{\bpi_k} = \sup_{s,\bu} \sum_{t=1}^{\infty}\gamma^t\rho^{\bpi}(s,\bu,t) $, the optimal solution $ p_k $ to a relaxation of optimization problem in Equation~\eqref{eq:ap_obj_1} satisfies relationship as follows:
	\begin{equation}
		p_k(s,\bu)=\frac{1}{Z^*}(D_k(s,\bu)+\epsilon_k(s,\bu)),
	\end{equation}
	where $ D_k(s,\bu)=d^{\bpi_k}(s,\bu)|Q_k-\mathcal{B}^*Q_{k-1}|\exp(-|Q_k-Q^*|)(1+\sum_{i=1}^{n}\prod_{j=1, j \neq i}^{n}\pi^j_k-n\prod_{i=1}^{n}\pi^i_k) $ and $ Z^* $ is the normalization constant.
\end{lemma}

\begin{proof}
	Suppose $ \bu^* \sim \bpi^*(s) $. Let $ \bpi=\bpi^* $ and $ \tilde{\bpi}=\bpi_k $ in Lemma~\ref{lem:kakade}, we have
	\begin{equation}
		\begin{aligned}
			&\eta(\bpi^*)-\eta(\bpi_k) \\
			&=-\frac{1}{1-\gamma}\mathbb{E}_{d^{\bpi_k}(s,\bu)}[A^{\bpi^*}(s,\bu)] \\
			&=\frac{1}{1-\gamma}\mathbb{E}_{d^{\bpi_k}(s,\bu)}[V^*(s)-Q^*(s,\bu)] \\
			&=\frac{1}{1-\gamma}\mathbb{E}_{d^{\bpi_k}(s,\bu)}[V^*(s)-Q_k(s,\bu^*)+Q_k(s,\bu^*)-Q_k(s,\bu)+Q_k(s,\bu)-Q^*(s,\bu)]\\
			&\stackrel{(a)}{\leq}\frac{1}{1-\gamma}\left[\mathbb{E}_{d^{\bpi_k}(s)}(Q^*(s,\bu^*)-Q_k(s,\bu^*))+\mathbb{E}_{d^{\bpi_k}(s,\bu)}(Q_k(s,\bu)-Q^*(s,\bu))+1\right] \\
			&\leq\frac{1}{1-\gamma}\left[\mathbb{E}_{d^{\bpi_k}(s)}|Q^*(s,\bu^*)-Q_k(s,\bu^*)|+\mathbb{E}_{d^{\bpi_k}(s,\bu)}|Q_k(s,\bu)-Q^*(s,\bu)|+1\right] \\
			&=\frac{2}{1-\gamma}\left[\mathbb{E}_{d^{\bpi_k,\bpi^*}}|Q^*(s,\bu)-Q_k(s,\bu)|+1\right]
		\end{aligned}
	\label{eq:upper}
	\end{equation}
	where $ d^{\bpi_k,\bpi^*}(s,\bu)=d^{\bpi_k}(s)\frac{\bpi_k+\bpi^*}{2}(\bu|s) $ and (a) uses Lemma~\ref{lem:inequal}.
\end{proof}

Since the original optimization is non-tractable, we consider this upper bound to obtain a closed-form solution. Therefore, we replace the objective in Equation~\eqref{eq:ap_obj_1} with the upper bound in Equation~\eqref{eq:upper} and solve the relaxed optimization problem, given by:
\begin{equation}
	\begin{aligned}
		\min_{p_k} \quad &\mathbb{E}_{d^{\bpi_k}(s,\mathbf{u})}[|Q_k-Q^*|(s,\mathbf{u})] \\
		\textrm{s.t.} \quad &Q_k=\arg\min_{Q \in \mathcal{Q}} \mathbb{E}_{p_k}[(Q-\mathcal{B}^*Q_{k-1})^2(s,\bu)], \\
		&\sum_{s,\bu}p_k(s,\bu)=1, \quad p_k(s,\bu) \ge 0,
	\end{aligned}
	\label{eq:ap_obj_2}
\end{equation}

As we cannot access $ \bpi^* $, we use $ d^{\bpi_k}(s,\bu) $ to replace $ d^{\bpi_k,\bpi^*} $. The best surrogate available is $ \bpi_k $. The objective in Equation~\eqref{eq:ap_obj_2} can be further relaxed with Jensen's inequality. Consider a convex function $ g(x) $ on the real space $ \mathbb{R} $, the inequality is given by:
\begin{equation}
	\mathbb{E}[g(X)] \ge g(\mathbb{E}[X]).
	\label{eq:jensen}
\end{equation} 
According to Equation~\eqref{eq:jensen}, we select the convex function $ g(x)=\exp(-x) $, and the objective can be further relaxed as:
\begin{equation}
	\begin{aligned}
		\min_{p_k} \quad &-\log\mathbb{E}_{d^{\bpi_k}(s,\mathbf{u})}[\exp(-|Q_k-Q^*|)(s,\mathbf{u})] \\
		\textrm{s.t.} \quad &Q_k=\arg\min_{Q \in \mathcal{Q}} \mathbb{E}_{p_k}[(Q-\mathcal{B}^*Q_{k-1})^2(s,\bu)], \\
		&\sum_{s,\bu}p_k(s,\bu)=1, \quad p_k(s,\bu) \ge 0,
	\end{aligned}
	\label{eq:ap_obj_3}
\end{equation}

In order to handle optimization problem in Equation~\eqref{eq:ap_obj_3}, we follow the standard procedures of Lagrangian multiplier method, which is:
\begin{equation}
	\mathcal{L}(p_k;\lambda,\psi)=-\log\mathbb{E}_{d^{\bpi_k}(s)}[\exp|Q_k-Q^*|(s,\mathbf{u})]+\lambda(\sum_{s,\mathbf{u}}p_k-1)-\psi\T p_k,
\end{equation}

After constructing the Lagrangian, we further compute some gradients that will be used in calculating the optimal solution. We first calculate the $ \frac{\p Q_k}{\p p_k} $ according to the implicit function theorem (IFT). Based on the first constraint in Equation~\eqref{eq:ap_obj_3}, we aim to find the minimum $ Q_k $ to satisfy the $ \arg\min(\cdot) $, and therefore we need to ensure the derivative of the term inside $ \arg\min(\cdot) $ (we use $ f(p_k,Q_k) $ to denote this term) to be zero, which is:
\begin{equation}
	f_{Q_k}'=2\sum_{\bu}p_k(Q_k-\mathcal{B}^*Q_{k-1})=0
\end{equation}

We can notice that $ F(p_k,Q_k):f_{Q_k}'=0 $ is an implicit function regarding $ Q_k $ and $ p_k $. Hence, we apply the IFT on the $ F(p_k,Q_k) $ considering the Hessian matrices of $ p_k $ and $ Q_k $ in $ f(p_k,Q_k) $ as follows:
\begin{equation}
	\frac{\p Q_k}{\p p_k}=-\frac{F_{p_k}'}{F_{Q_k}'}=-[\mathrm{diag}(p_k)]^{-1}[\mathrm{diag}(Q_k-\mathcal{B}^*Q^*_{k-1})].
	\label{eq:ift}
\end{equation}

Next, we derive the expression for $ \frac{\p d^{\bpi_k}(s,\bu)}{\p p_k} $ in the following equation:
\begin{equation}
	\begin{aligned}
		\frac{\p d^{\bpi_k}(s,\bu)}{\p p_k}&=\frac{\p d^{\bpi_k}(s,\bu)}{\p \bpi_k}\frac{\p \bpi_k}{\p Q_k}\frac{\p Q_k}{\p p_k} \\
		&=\mathrm{diag}\left(d^{\bpi_k}(s)\prod_{\substack{j=1 \\ j \neq i}}^{n}\pi_k^j+\epsilon_0(s)\right)\frac{\p \bpi_k}{\p Q_k}\frac{\p Q_k}{\p p_k} \\
		&\stackrel{(b)}{=}\mathrm{diag}\left(d^{\bpi_k}(s)\prod_{\substack{j=1 \\ j \neq i}}^{n}\pi_k^j+\epsilon_0(s)\right)\mathrm{diag}\left(\frac{\p \pi_k^i}{\p Q_k}\right)\frac{\p Q_k}{\p p_k} \\
		&\stackrel{(c)}{=}\mathrm{diag}\left(d^{\bpi_k}(s)\prod_{\substack{j=1 \\ j \neq i}}^{n}\pi_k^j+\epsilon_0(s)\right)\prod_{i=1}^{n}\pi_k^i\mathrm{diag}(1-\pi_k^i)\frac{\p Q_k}{\p p_k} \\
		&=d^{\bpi_k}(s,\bu)\left(\sum_{i=1}^{n}\prod_{\substack{j=1 \\ j \neq i}}^{n}\pi^j_k-n\prod_{i=1}^{n}\pi^i_k\right)\frac{\p Q_k}{\p p_k}+\epsilon_0(s)\prod_{i=1}^{n}\pi_k^i\mathrm{diag}(1-\pi_k^i)\frac{\p Q_k}{\p p_k},
	\end{aligned}
	\label{eq:part_gard}
\end{equation}
where $ \epsilon_0(s)=\frac{\p d^{\bpi_k}(s,\bu)}{\p \bpi_k} $ is a small quantity provided by Lemma~\ref{lem:d_pi}. Besides, (b) and (c) are based on the the definition of the Boltzmann policy and Assumption~\ref{assum:policy}.

Since we have all the preparations ready, we now compute the Lagrangian by applying the Karush–Kuhn–Tucker (KKT) condition. We let the Lagrangian gradient to be zero, i.e.,
\begin{equation}
	\frac{\p \mathcal{L}(p_k;\lambda,\psi)}{\p p_k} = 0
	\label{eq:kkt_1}
\end{equation} 

Besides, the partial derivative of the Lagrangian can be computed as:
\begin{equation}
	\begin{aligned}
		\frac{\p \mathcal{L}(p_k;\lambda,\psi)}{\p p_k}&=-\frac{\p \log\mathbb{E}_{d^{\bpi_k}(s,\mathbf{u})}[\exp(-|Q_k-Q^*|)(s,\mathbf{u})]}{\p p_k} +\lambda-\psi_{s,\bu} \\
		&=\frac{1}{Z}\exp(-|Q_k-Q^*|)(\frac{\p d^{\bpi_k}(s,\bu)}{\p p_k} + d^{\bpi_k}(s,\bu)\frac{\p Q_k}{\p p_k})+\lambda-\psi_{s,\bu},
	\end{aligned}
	\label{eq:kkt_2}
\end{equation}
where $ Z=\mathbb{E}_{s',\bu'\sim d^{\bpi_k}(s,\bu)}\exp(-|Q_k-Q^*|)(s',\bu') $.

Based on Equation~\eqref{eq:kkt_1} and~\eqref{eq:kkt_2}, and substituting the expression of $ \frac{\p Q_k}{\p p_k} $ and $ \frac{\p d^{\bpi_k}(s,a)}{\p p_k} $ with the derived results in Equation~\eqref{eq:ift} and~\eqref{eq:part_gard}, we obtain:
\begin{equation}
	\begin{aligned}
		p_k(s,a)=&\frac{1}{Z(\psi_{s,\bu}^*-\lambda^*)}\left[d^{\bpi_k}(s,\bu)|Q_k-\mathcal{B}^*Q_{k-1}|\exp(-|Q_k-Q^*|)\left(1+\sum_{i=1}^{n}\prod_{\substack{j=1 \\ j \neq i}}^{n}\pi^j_k-n\prod_{i=1}^{n}\pi^i_k\right) \right.  \\
		&\left.+\epsilon_0|Q_k-\mathcal{B}^*Q_{k-1}|\exp(-|Q_k-Q^*|)\prod_{i=1}^{n}\pi_k^i\mathrm{diag}(1-\pi_k^i)\right],
	\end{aligned}
	\label{eq:kkt_3}
\end{equation}

According to Lemma~\ref{lem:d_pi}, the value of $ \epsilon_0 $ is smaller than $ d^{\bpi_k}(s) $ so the second term will not influence the sign of the equation. Equation~\eqref{eq:kkt_3} will always be larger or equal to zero. By KKT condition, when Equation~\eqref{eq:kkt_3} equal to zero, we let $ \psi_{s,\bu}^*=0 $ because the value of $ \psi_{s,\bu}^* $ will not affect $ p_k $. In the contrast, when Equation~\eqref{eq:kkt_3} is larger than 0, the $ p_k $ should equal to zero. Therefore, Equation~\eqref{eq:kkt_3} can be simplify as follows:
\begin{equation}
	p_k(s,\mathbf{u})=\frac{1}{Z^*}(D_k(s,\mathbf{u})+\epsilon_k(s,\mathbf{u})),
\end{equation}
where we have
\begin{equation}
	\begin{aligned}
		&D_k(s,\bu)=d^{\bpi_k}(s,\bu)|Q_k-\mathcal{B}^*Q_{k-1}|\exp(-|Q_k-Q^*|)\left(1+\sum_{i=1}^{n}\prod_{\substack{j=1 \\ j \neq i}}^{n}\pi^j_k-n\prod_{i=1}^{n}\pi^i_k\right), \\
		&\epsilon_k=\epsilon_0|Q_k-\mathcal{B}^*Q_{k-1}|\exp(-|Q_k-Q^*|)\prod_{i=1}^{n}\pi_k^i\mathrm{diag}(1-\pi_k^i)
	\end{aligned}
\end{equation}


This concludes the proof.

\section{Proof of Theorem~\ref{theo:cond}}
\label{proof:theo_2}

We first reform Equation~\eqref{eq:joint_prob} via factorization of extracting $ \pi_k^a $, given by:
\begin{equation}
	f = 1+\prod_{\substack{i=1 \\ i \neq a}}^{n}\pi^i_k+\pi_k^a(\sum_{\substack{i=1 \\ i \neq a}}^{n}\prod_{\substack{j=1 \\ j \neq i,a}}^{n}\pi_k^i\pi_k^j-n\prod_{\substack{i=1 \\ i \neq a}}^{n}\pi^i_k).
	\label{eq:fac}
\end{equation}

Let $ g $ represent the term $ \sum_{i=1,i \neq a}^{n}\prod_{j=1,j \neq i,a}^{n}-n\prod_{i=1,i \neq a}^{n}\pi^i_k $ in Equation~\eqref{eq:fac}. Based on the fact that the agent's action probability is within the range $ [0,1] $, if $ g > 0 $, we have $ \pi_k^a = 1 $ to ensure the $ f $ reaches its maximum; if $ g < 0 $, we let $ \pi_k^a = 0 $ for the same purpose. In the case that $ g=0 $, the value of $ \pi_k^a $ can be either 0 or 1, as the third term, including $ g $ will be eventually canceled out, and the value of $ \pi_k^a $ will not affect $ f_{\rm max} $. Therefore, it is obvious that the condition for reaching $ f_{\rm max} $ is that all the values of $ \pi_k^a $ must be on their boundary.

After determining the boundary condition, from Equation~\eqref{eq:fac}, the second term, which is $ \prod_{i=1,i \neq a}^{n}\pi^i_k $, will be either 0 or 1. If it equals 1, indicating all the $ \pi_k^i $ with $ i \neq a $ is 1; if the second term equals 0, at least one of the agents' probabilities $ \pi_k^i $ is 0. Assuming that the number of $ \pi_k^i $ equaling 0 is $ N_{\pi} $, we have:
\begin{equation}
	g=
	\begin{cases}
			-1 &N_{\pi} = 0 \\
			1 &N_{\pi} = 1\\
			0 &N_{\pi} \geq 2.
		\end{cases}
	\label{eq:g}
\end{equation} 

According to Equation~\eqref{eq:g}, we can numerically compute $ f_{\rm max} $ as follows:
\begin{equation}
	f_{\rm max}=
	\begin{cases}
			2, &N_{\pi} = 0, \pi_k^a=0 \text{ or } N_{\pi}=1, \pi_k^a=1 \\
			1, &N_{\pi} \geq 2,\forall \pi_k^a.
		\end{cases}
\end{equation}

Since $ \pi_k^a $ can be the probability of any selected agent, such discussion is applicable for the probabilities of all agents in the environment. So far, we successfully proved the Theorem~\ref{theo:cond} that, to maximize the joint action probability function $ f $ without loss of the generality, we shall let all $ \pi_k^i $ equal to 1, but at least one of the probabilities $ \pi_k^a $ be 0.

This concludes the proof.

\section{Algorithms}
\label{subsec:alg}
In this section, we provide the pseudo-codes for \name and \name Approximation in Algorithms~\ref{alg:remer} and \ref{alg:remer_approx}, respectively.

\begin{algorithm}
	\caption{\name}
	\label{alg:remer}
	\begin{algorithmic}[1]
		\STATE Initialize step, learning rate $ \alpha $ and replay buffer $ \mathcal{D} $, and set $ \theta^- = \theta $
		\FOR{$ \text{step} =1:\text{step}_{max} $ }
		\STATE $ k=0, s_0 = $ initial state
		\WHILE{$ s_k \neq $ terminal and $ k< $ episode limit}
		\FOR{each agent $ a $}
		\STATE $ \tau^a_k = \tau^a_{k-1} \cup {(o_k, u_{k-1})} $
		\STATE $ u^a_k = 
		\begin{cases}
			\arg\max_{u^a_k}Q(\tau^a_k,u^a_k) & \text{with probability } 1 - \epsilon \\
			\mathrm{randint}(1,|U|) & \text{with probability } \epsilon
		\end{cases} $
		\ENDFOR
		\STATE Obtain the reward $ r_k $ and next state $ s_{k+1} $
		\STATE Store the current trajectory into replay buffer $ \mathcal{D} = \mathcal{D} \cup {(s_k,\bu_k, r_k, s_{k+1})} $
		\STATE $ k = k + 1, \text{step}=\text{step}+1 $
		\ENDWHILE
		\STATE Collect $ b $ samples from the replay buffer $ \mathcal{D} $ following uniform distribution $ \mu $.
		\FOR{each timestep $ k $ in each episode in batch $ b $}
		\STATE Evaluate $ Q_k $, $ Q^* $ and target values
		\STATE Obtain the utilities $ Q_a $ from agents' local networks, and compute the individual policy $ \pi_k^a $
		\STATE Compute the weight: \\ $ w_k \propto |Q_k-\mathcal{B}^*Q_{k-1}|\exp(-|Q_k-Q^*|)\left(1+\sum_{i=1}^{n}\prod_{j=1, j \neq i}^{n}\pi^j_k-n\prod_{i=1}^{n}\pi^i_k\right) $
		\ENDFOR
		\STATE Minimize the Bellman error for $ Q_k $ weighted by $ w_k $, update the network parameter $ \theta $: \\ $ \theta = \theta - \alpha(\nabla_\theta \frac{1}{b} \sum_{i}^{b}w_k(Q_k -y_i)^2) $.
		\IF{update-interval steps have passed}
		\STATE $ \theta^-=\theta $
		\ENDIF
		\ENDFOR
	\end{algorithmic}
\end{algorithm}


\begin{algorithm}
	\caption{\name Approximation}
	\label{alg:remer_approx}
	\begin{algorithmic}[1]
		\STATE Initialize step, learning rate $ \alpha $ and replay buffer $ \mathcal{D} $, and set $ \theta^- = \theta $
		\FOR{$ \text{step} =1:\text{step}_{max} $ }
		\STATE $ k=0, s_0 = $ initial state
		\WHILE{$ s_k \neq $ terminal and $ k< $ episode limit}
		\FOR{each agent $ a $}
		\STATE $ \tau^a_k = \tau^a_{k-1} \cup {(o_k, u_{k-1})} $
		\STATE $ u^a_k = 
		\begin{cases}
			\arg\max_{u^a_k}Q(\tau^a_k,u^a_k) & \text{with probability } 1 - \epsilon \\
			\mathrm{randint}(1,|U|) & \text{with probability } \epsilon
		\end{cases} $
		\ENDFOR
		\STATE Obtain the reward $ r_k $ and next state $ s_{k+1} $
		\STATE Store the current trajectory into replay buffer $ \mathcal{D} = \mathcal{D} \cup {(s_k,\bu_k, r_k, s_{k+1})} $
		\STATE $ k = k + 1, \text{step}=\text{step}+1 $
		\ENDWHILE
		\STATE Collect $ b $ samples from the replay buffer $ \mathcal{D} $ following uniform distribution $ \mu $.
		\FOR{each timestep $ k $ in each episode in batch $ b $}
		\STATE Evaluate $ Q_k $, $ Q^* $ and target values
		\STATE Obtain the utilities $ Q_a $ from agents' local networks, and compute the individual policy $ \pi_k^a $
		\STATE Re-scale the weights to high $ \alpha_h $, medium $ \alpha_m $, and low $ \alpha_l $ based on one individual policy $ \pi_k^a $ and other policies $ \pi_k^{-a} $:
		\STATE Compute the weight: \\ $ w_k \propto |Q_k-\mathcal{B}^*Q_{k-1}|\exp(-|Q_k-Q^*|)\cdot
		\begin{cases}
			\alpha_h & \text{when } \pi_k^a \approx 0 $ and $ \pi_k^a \ll \pi_k^{-a} \\
			\alpha_l & \text{when } \prod_{i=1}^{n}\pi_k^i \approx 0 $, or $ \prod_{i=1}^{n}\pi_k^i \approx 1 \\
			\alpha_m & \text{elsewhere}
		\end{cases} $
		\ENDFOR
		\STATE Minimize the Bellman error for $ Q_k $ weighted by $ w_k $, update the network parameter $ \theta $: \\ $ \theta = \theta - \alpha(\nabla_\theta \frac{1}{b} \sum_{i}^{b}w_k(Q_k -y_i)^2) $.
		\IF{update-interval steps have passed}
		\STATE $ \theta^-=\theta $
		\ENDIF
		\ENDFOR
	\end{algorithmic}
\end{algorithm}

\section{Environment Details}
\label{subsec:setup}

We use more recent baselines (i.e., FOP and DOP) that are known to outperform QTRAN \cite{son2019qtran} and QPLEX \cite{wang2020qplex} in the evaluation. In general, we tend to choose baselines that are more closely related to our work and most recent. This motivated the choice of QMIX (baseline for value-based factorization methods), WQMIX (close to our work that uses weighted projections so better joint actions can be emphasized), VDAC \cite{su2021value}, FOP \cite{zhang2021fop}, DOP \cite{wang2020dop} (SOTA actor-critic based methods). We acquired the results of QMIX, WQMIX based on their hyper-parameter tuned versions from pymarl2\cite{hu2021riit} and implemented our algorithm based on it.

\subsection{Predator-Prey}
A partially observable environment on a  grid-world predator-prey task is used to model relative overgeneralization problem \cite{bohmer2020deep} where 8 agents have to catch 8 prey in a 10 × 10 grid. Each agent can either move in one of the 4 compass directions, remain still, or try to
catch any adjacent prey. Impossible actions, i.e., moving into an occupied target position or catching when there is no adjacent prey, are treated as unavailable. If two adjacent agents execute the catch action, a prey is caught and both the prey and the catching agents are removed from the grid. An agent’s observation is a 5 × 5 sub-grid centered around it, with one channel showing agents and another indicating prey.  An episode ends if all agents have been removed or after 200 steps. Capturing a prey is rewarded with r = 10, but unsuccessful attempts by single agents are punished by a negative reward p. In this paper, we consider two sets of experiments with $p$ = (0, -0.5, -1.5, -2). The task is similar to the matrix game proposed by \cite{son2019qtran} but significantly more complex, both in terms of the optimal policy and in the number of agents. 

\subsection{SMAC}
For the experiments on StarCraft II micromanagement, we follow the setup of SMAC \cite{samvelyan2019starcraft} with open-source implementation including QMIX \cite{rashid2018qmix}, WQMIX \cite{rashid2020weighted}, QPLEX \cite{wang2020qplex}, FOP \cite{zhang2021fop}, DOP \cite{wang2020dop} and VDAC \cite{su2021value}. We consider combat scenarios where the enemy units are controlled by the StarCraft II built-in AI and the friendly units are controlled by the algorithm-trained agent. The possible options for built-in AI difficulties are Very Easy, Easy, Medium, Hard, Very Hard, and Insane, ranging from 0 to 7. We carry out the experiments with ally units controlled by a learning agent while built-in AI controls the enemy units with difficulty = 7 (Insane). Depending on the specific scenarios(maps), the units of the enemy and friendly can be symmetric or asymmetric. At each time step each agent chooses one action from discrete action space, including noop, move[direction], attack[enemy\_id], and stop. Dead units can only choose noop action. Killing an enemy unit will result in a reward of 10 while winning by eliminating all enemy units will result in a reward of 200. The global state information is only available in the centralized critic. Each baseline algorithm is trained with 4 random seeds and evaluated every 10k training steps with 32 testing episodes for main results, and with 3 random seeds for ablation results and additional results.

\subsection{Implementation Details and Hyperparameters}

\begin{table}[h]
	\centering
	\caption{Hyperparameter value settings.}
	\label{tab:hyper}
	\begin{tabular}{lc}
		\toprule
		Hyperparameter & Value \\
		\midrule 
		Batch size & 128 \\
		Replay buffer size & 10000 \\
		Target network update interval & Every 200 episodes \\
		Learning rate & 0.001 \\
		TD-lambda & 0.6 \\
		\bottomrule
	\end{tabular}
\end{table}

In this section, we introduce the implementation details and hyperparameters we used in the experiment. We carried out the experiments on NVIDIA 2080Ti with fixed hyperparameter settings.
Recently \cite{hu2021riit} demonstrated that MARL algorithms are significantly influenced by code-level optimization and other tricks, e.g. using TD-lambda, Adam optimizer, and grid-searched/Bayesian optimized~\cite{mei2022bayesian} and hyperparameters (where many state-of-the-art are already adopted), and proposed fine-tuned QMIX and WQMIX, which is demonstrated with significant improvements from their original implementation. We implemented our algorithm based on its open-sourced codebase and acquired the results of QMIX and WQMIX from it. 

We use one set of hyperparameters for each environment, i.e., no tuned hyperparameters for individual maps. We use epsilon greedy for action selection with annealing from $\epsilon$ = 0.995 decreasing to $\epsilon$ = 0.05 in 100000 training steps in a linear way. The performance for each algorithm is evaluated for 32 episodes every 1000 training steps. Additional hyperparameter values are provided in Table~\ref{tab:hyper}.


\end{document}